\documentclass[11pt]{article}
\usepackage{amsmath,amsthm,amssymb}   

\usepackage{algorithm,algorithmic}

\usepackage[bottom]{footmisc}
\usepackage{macros}

\usepackage{subcaption}
\usepackage{graphicx}
\usepackage{pdfpages}
\usepackage{eufrak}
\usepackage[margin=1 in]{geometry}
\allowdisplaybreaks

{\makeatletter
 \gdef\xxxmark{      \protected@write\@auxout{\def\PAGE{ page }}
     {\@percentchar xxx: section \thesubsubsection \PAGE \thepage}   \expandafter\ifx\csname @mpargs\endcsname\relax      \expandafter\ifx\csname @captype\endcsname\relax        \marginpar{xxx}     \else
       xxx      \fi
-   \else
     xxx    \fi}
 \gdef\xxx{\@ifnextchar[\xxx@lab\xxx@nolab}
 \long\gdef\xxx@lab[#1]#2{{\bf [\xxxmark #2 ---{\sc #1}]}}
 \long\gdef\xxx@nolab#1{{\bf [\xxxmark #1]}}
   \gdef\turnoffxxx{\long\gdef\xxx@lab[##1]##2{}\long\gdef\xxx@nolab##1{}}}

\def\<{\langle}
\def\>{\rangle}

\def\R{\mathbb{R}}

                                                           \DeclareMathOperator*{\Exp}{\mathbb E}

\def\shownotes{1}  \ifnum\shownotes=1
\newcommand{\authnote}[2]{{$\ll$\textsf{\footnotesize #1 notes: #2}$\gg$}}
\else
\newcommand{\authnote}[2]{}
\fi

\DeclareMathOperator{\Id}{Id}

\newcommand{\hide}[1]{}

\newcommand{\veps} {\varepsilon}

\newcommand{\vecgt}{{u^{\star}}}

\newcommand{\bigabs}[1]{\left|{#1}\right|}
\newcommand{\normFro}[1]{\norm{#1}_{_{F}}}
\newcommand{\normNuclear}[1]{\norm{#1}_{*}}
\newcommand{\bignorm}[1]{\left\|#1\right\|}

\newcommand{\bignormFro}[1]{\bignorm{#1}_{F}}
\newcommand{\innerProduct}[2]{\langle#1, #2\rangle}

\theoremstyle{plain}
\newtheorem{thm}{Theorem}[section]

\newtheorem{prop}[thm]{Proposition}
\newtheorem{cor}[thm]{Corollary}
\newtheorem{lem}[thm]{Lemma}
\newtheorem{claim}[thm]{Claim}

\newtheorem{defn}[thm]{Definition}

\theoremstyle{remark}

\numberwithin{equation}{section}

\newtheorem*{rep@theorem}{\rep@title}

\author{Yuanzhi Li\thanks{Princeton University. yuanzhil@cs.princeton.edu}\and Tengyu Ma \thanks{Facebook AI Research. tengyuma@stanford.edu}\and Hongyang Zhang\thanks{Stanford University. hongyang@cs.stanford.edu}}
\begin{document}

\title{Algorithmic Regularization in Over-parameterized Matrix Sensing and Neural Networks with Quadratic Activations}

\maketitle

\begin{abstract}
    We show that the gradient descent algorithm provides an implicit regularization effect in the learning of over-parameterized matrix factorization models and one-hidden-layer neural networks with quadratic activations. 
    
Concretely, we show that given $\tilde{O}(dr^{2})$ random linear measurements of a rank $r$ positive semidefinite matrix $X^{\star}$, we can recover $X^{\star}$ by parameterizing it by $UU^\top$ with $U\in \mathbb R^{d\times d}$ and minimizing the squared loss, even if $r \ll d$. We prove that starting from a small initialization, gradient descent recovers $X^{\star}$ in $\tilde{O}(\sqrt{r})$ iterations approximately. The results solve the conjecture of Gunasekar et al.~\cite{gunasekar2017implicit} under the restricted isometry property. 
    
The technique can be applied to analyzing neural networks with one-hidden-layer quadratic activations with some technical modifications. 
\end{abstract}

\section{Introduction}

Over-parameterized models are crucial in deep learning, but their workings are far from understood. Over-parameterization --- the technique of using more parameters than statistically necessary --- apparently improves the training: theoretical and empirical results have suggested that it can enhance the geometric properties of the optimization landscape in simplified settings~\cite{livni2014computational,hardt2016gradient,hardt17identity,soudry2016no} and thus make it easier to train over-parameterized models.

On the other hand,  over-parameterization often doesn't hurt the test performance, even if the number of parameters is much larger than the number of examples. Large neural networks used in practice have enough expressiveness to fit any labels of the training datasets~\cite{zhang2016understanding,hardt17identity}. The training objective function may have multiple global minima with almost zero training error, some of which generalize better than the others~\cite{keskar2016large,dinh2017sharp}. However, local improvement algorithms such as stochastic gradient descent, starting with proper initialization, may prefer some generalizable local minima to the others and thus provide an implicit effect of regularization~\cite{srebro2011universality,neyshabur2014search,hardt2015train,neyshabur2017exploring,wilson2017marginal}. 
Such regularization seems to depend on the algorithmic choice, the initialization scheme, and certain intrinsic properties of the data.

The phenomenon and intuition above can be theoretically fleshed out in the context of linear models~\cite{soudry2017implicit}, whereas less is known for non-linear models whose training objectives are usually non-convex. The very important work of Gunasekar et al.~\cite{gunasekar2017implicit} initiates the study of low-rank matrix factorization models with over-parameterization and conjectures that gradient descent prefers small trace norm solution in over-parameterized models with thorough empirical evidences.

This paper resolves the conjecture for the matrix sensing problem --- recovering a low-rank matrix from linear measurements --- under the restricted isometry property (RIP). We show that with a full-rank factorized parameterization, gradient descent on the squared loss with finite step size, starting with a small initialization, converges to the true low-rank matrix (which is also the minimum trace norm solution.) One advantage of the over-parameterized approach is that without knowing/guessing the correct rank,  the algorithms can automatically pick up the minimum rank or trace norm solution that fits the data. 

The analysis can be extended to learning one-hidden-layer neural networks with quadratic activations. We hope such theoretical analysis of algorithmic regularization in the non-convex setting may shed light on other more complicated models where over-parameterization is crucial (if not necessary) for efficient training.

\subsection{Setup and Main Results}
Let $\bXg$ be an unknown rank-$r$ symmetric positive semidefinite (PSD) matrix in $\R^{d\times d}$ that we aim to recover. Let $\bA_1, \cdots, \bA_m \in \mathbb{R}^{d \times d}$ be $m$ given symmetric measurement matrices.\footnote{Given that the matrix $\bXg$ is symmetric, we can assume that $A_i$'s are symmetric without loss of generality: Because $\inner{A_i, \bXg} = \inner{\frac{1}{2}(A_i+A_i^\top), \bXg}$ for any symmetric matrix $\bXg$, we can always replace $A_i$ by $\frac{1}{2}(A_i+A_i^\top)$.} We assume that the label vector $y\in \R^m$ is generated by linear measurements 
$$y_i = \langle \bA_i, \bXg \rangle.$$
Here $\inner{A, B} = \trace(A^\top B)$ denotes the inner product of two matrices. Our goal is to recover the matrix $\bXg$. \footnote{Our analysis can naturally handle a small amount of Gaussian noise in the label vector $y$, but for simplicity we only work with the noiseless case. }

Without loss of generality, we assume that $\bXg$ has spectral norm 1.  Let $\sigma_{r}(\bX)$ denote the $r-th$ singular value of a matrix $\bX$, and let $\kappa = 1/\sigma_{r}(\bXg)$ be the condition number of $\bXg$. We focus on the regime where $r\ll d$ and $m \approx d\cdot\poly(r \log d) \ll d^2$.

Let $U\in \mathbb{R}^{d\times d}$ be a matrix variable. We consider the following mean squared loss objective function with over-parameterization:
\begin{align}
\min_{U \in \mathbb{R}^{d \times d}} f(\bU) = \frac{1}{2m}\sum_{i = 1}^m \left(y_i- \langle \bA_i, \bU \bU^{\top} \rangle\right)^2\label{eqn:obj}
\end{align}
Since the label is generated by $y_i = \inner{A_i, \bXg}$, any matrix $U$ satisfying $UU^\top = \bXg$ is a local minimum of $f$ with zero training error. These are the ideal local minima that we are shooting for.  However, because the number of parameters $d^2$ is much larger than the number of observation $m$, there exist other choices of $U$ satisfying $f(U) = 0$ but $UU^\top \neq \bXg$. 

A priori, such over-parameterization will cause over-fitting. However, we will show that the following gradient descent algorithm with small initialization converges to a desired local minimum, instead of other non-generalizable local minima:
\begin{align}
& U_0 = \alpha B, \textup{   where $B\in \R^{d\times d}$ is any orthonormal matrix}\nonumber\\
& U_{t+1} = U_t - \eta \nabla f(U_t)\label{eqn:init}
\end{align}
The following theorem assumes the measurements matrices $A_1,\dots, A_m$ satisfy restricted isometry property (RIP),  which is formally defined in Section~\ref{sec:prelim}. Casual readers may simply assume that the entries of $A_i$'s are drawn i.i.d from standard normal distribution\footnote{Or equivalently, as discussed in the previous footnote, causal readers may assume $A_i = \frac{1}{2}(Q_i+Q_i^\top)$ where $Q_i$ is from standard normal distribution. Such symmetrization doesn't change the model since $\inner{Q_i, \bXg} = \inner{\frac{1}{2}(Q_i+Q_i^\top), \bXg}$ } and the number of observations $m \lesssim dr^{2}\log^3 d$: it's known~\cite{recht2010guaranteed} that in this case $A_1,\dots, A_m$ meet the requirement of the following theorem, that is, they satisfy $(4r,\delta)$-RIP with $\delta \lesssim 1/(\sqrt{r}\log d)$ with high probability.
\footnote{Technically, to get such RIP parameters that depends on $r$, one would need to slightly modify the proof of~\cite[Theorem 4.2]{recht2010guaranteed} at the end to get the dependency of $m$ on $\delta$. } 

\begin{thm}\label{thm:intro-main}
	Let $c$ be a sufficiently small absolute constant. Assume that the set of measurement matrices $(A_1,\dots, A_m)$ satisfies $(4r,\delta)$-restricted isometry property (defined in Section~\ref{sec:prelim} formally) with $\delta \le c/(\kappa^3\sqrt{r}\log^2 d)$. Suppose the initialization and learning rate satisfy $0< \alpha \le c\min\{\delta\sqrt{r}\kappa, 1/d\}$ and $\eta \le c\delta$. Then for \emph{every} $(\kappa\log (\frac d {\alpha}))/\eta\lesssim T \lesssim  1/(\eta\sqrt{d\kappa \alpha})$, we have
	$$\norm{U_TU_T^\top-\bXg}_F^2 \lesssim {\alpha} \sqrt d / \kappa^2. $$
\end{thm}

Note that the recovery error $\norm{U_TU_T^\top-\bXg}_F^2$ can be viewed as the test error (defined in Equation \eqref{eq_test_error} formally) --- it's the expectation of the test error on a fresh measurement $A_j$ drawn from the standard normal distribution. The theorem above shows that gradient descent can provide an algorithmic regularization so that the generalization error depends on the size of the initialization $\alpha$, instead of the number of parameters. Because the convergence is not very sensitive to the initialization, we can choose small enough $\alpha$ (e.g., $1/d^{5}$) to get approximately zero generalization error. Moreover, when $\alpha$ is small, gradient descent can run for a long period of time without overfitting the data. We show in Section~\ref{sec:exp} that empirically indeed the generalization error depends on the size of the initialization and gradient descent is indeed very stable. 

The analysis also applies to \textit{stochastic gradient descent}, as long as each batch of the measurement matrices satisfies RIP.\footnote{Smaller batch size should also work when the learning rate is sufficiently small, although its analysis seems to require more involved techniques and is left for future work.} We also remark that our theory suggests that early stopping for this problem is not necessary when the initialization is small enough --- the generalization error bounds apply until $1/(\eta\sqrt{d\kappa \alpha})$ iterations. We corroborate this with empirical results in Section~\ref{sec:exp}. 

We remark that we achieve a good iteration complexity bound $1/\eta \approx 1/\delta\approx \sqrt{r}$ for the gradient descent algorithm, which was not known in previous work even for low-rank parameterization, nor for the case with infinite samples (which is the PCA problem).   Part of the technical challenges is to allow finite step size $\eta$ and inverse-poly initialization $\alpha$ (instead of exponentially small initialization). The dependency of $\delta$ on $\kappa$ and $r$ in the theorem is possibly not tight. We conjecture that $\delta$ only needs to be smaller than an absolute constant, which is left for future work. 

\paragraph{Insights of the analysis:} Interestingly, our analysis ideas seem to be different from other previous work in a conceptual sense. The analysis of the logistic regression case~\cite{soudry2017implicit} relies on that the iterate eventually moves to infinity. The folklore analysis of the algorithmic regularization of SGD for least squares and the analysis in~\cite{gunasekar2017implicit} for the matrix regression with commutable measurements both follow the two-step plan: a) the iterates always stays on a low-rank manifold that only depends on the inputs (the measurement matrices) but not on the label vector $y$; b) generalization follows from the low complexity of the low-rank manifold. Such input-dependent but label-independent manifold doesn't seem to exist in the setting when $A_i$'s are random. 

Instead, we show that the iterates stay in the set of matrices with approximate rank smaller or equal to the minimal possible rank that can fit the data, which is a set that depends on the labels $y$ but not on the inputs $A_i$'s.  
We implicitly exploit the fact that gradient descent on the population risk with small initialization only searches through the space of solutions with a \textit{lower} rank than that of the true matrix $\bXg$.  The population risk is close to the empirical risk on matrices with rank smaller than or equal to the true rank.  Hence, we can expect the learning dynamics of the empirical risk  to be similar to that of the population risk, and therefore the iterates of GD on the empirical risk remain approximately low-rank as well. Generalization then follows straightforwardly from the low-rankness of the iterates. See Section~\ref{sec:rank1} for more high-level discussions.

We note that the factorized parameterization also plays an important role here. The intuition above would still apply if we replace $UU^\top$ with a single variable $X$ and run gradient descent in the space of $X$ with small enough initialization. However, it will converge to a solution that \textit{doesn't} generalize.  The discrepancy comes from another crucial property of the factorized parameterization: it provides certain denoising effect that encourages the empirical gradient to have a smaller eigenvalue tail. This ensures the eigenvalues tails of the iterates to grow sufficiently slowly. This point will be more precise in Section~\ref{sec:rank1} once we apply the RIP property. In section~\ref{sec:exp}, we also empirically demonstrate that GD in the original space of $\bXg$ with projection to the PSD cone doesn't provide as good generalization performance as GD in the factorized space.  

Finally, we remark that the cases with rank $r> 1$ are technically much more challenging than the rank-1 case. For the rank-1 case, we show that the spectrum of $U_t$ remains small in a fixed rank-$(d-1)$ subspace, which is exactly the complement of the column span of $\bXg$. Hence the iterates are approximately rank one. By contrast, for the rank-$r$ case, a direct extension of this proof strategy only gives a much weaker result compared to Theorem \ref{thm:intro-main}. Instead, we identify an \textit{adaptive}  rank-$(d-r)$ subspace in which $U_t$ remains small. Clearly, the best choice of this adaptive subspace is the subspace of the least $(d-r)$ left singular vectors of $U_t$. However, we use a more convenient surrogate. We refer the reader to Section~\ref{sec:mainproof} for detailed descriptions.

\subsection{Extensions to Neural Networks with Quadratic Activations}

Our results can be applied to learning one-hidden-layer neural networks with quadratic activations. We setup the notations and state results below and defer the details to Section~\ref{sec:quadratic}. 

Let $x\in \R^d$ be the input and $U^\star\in \R^{d\times r}$ be the first layer weight matrix. We assume that the weight on the second layer is simply the all one's vector $\mathbf{1}\in \R^r$.  Formally, the label $y$ is assumed to be generated by 
\begin{align}
y = \mathbf{1}^\top q({U^\star}^\top x)  \label{eqn:qnn}
\end{align}
where $q(\cdot)$ is the element-wise quadratic function. For simplicity, we assume that $x$ comes from standard normal distribution $\mathcal{N}(0,\Id_{d\times d})$. It's not hard to see that the representational power of the hypothesis class with $r=d$ is the same as those with $r > d$. Thus we only focus on the case when $r \le d$.
For the purpose of this paper, the most interesting regime is the scenario when $r\ll d$. 

We use an over-parameterized model with a variable $U\in \R^{d\times d}$. The prediction $\hat{y}$ is parameterized by $\hat{y} = \mathbf{1}^\top q(U^\top x) $, 
and we use the mean squared error $(y-\hat{y})^2$ as the loss function. We use a variant of stochastic gradient descent (or gradient descent) on the mean squared loss. 

The following theorem shows that the learned model will generalize with $\tilde{O}(dr^{5} \kappa^6)$ examples, despite that the number of parameters $d^2$ can be much larger than the number of samples (when $d \gg r$ or $r$ is considered as a constant).\footnote{The dependency on $r$ here is likely to be loose. Although we note that this is the first bound of this kind for this problem that shows over-parameterized models can still generalize. } We will start with an initialization $U_0$ in the same way as in equation~\eqref{eqn:init}, and denote $U_1,\dots, U_T$ as the iterates. Let $\kappa$ be the condition number of $\bUg {\bUg}^\top$.

\begin{thm}\label{thm:main-quadratic}
Given $\tilde{O}(dr^5\kappa^6)$ examples, 	a variant of gradient descent (Algorithm~\ref{alg:aqnn} in Section~\ref{sec:quadratic}) with initialization $\alpha \lesssim \min\{1/d, 1/(r^2\kappa^4\log^2 d)\}$ and learning rate $\eta \lesssim \frac{1}{\kappa^3 r^{1.5} \log^2 d}$ returns a solution with generalization error at most $O(d\kappa \alpha)$ at any iteration $t$ such that  $(\kappa\log (d/\alpha))/\eta \lesssim t \lesssim 1/(\eta\sqrt{d\kappa \alpha})$. 
\end{thm}

\noindent The same analysis also applies to stochastic gradient descent as long as the batch size is at least $\gtrsim dr^5\kappa^6$. The analysis exploits the connection (pointed out by ~\cite{2017arXiv170704926S}) between neural networks with quadratic activations and matrix sensing with rank-1 measurements~\cite{kueng2017low, zhong2015efficient,chen2015exact}: one can view $xx^\top$ as the measurement matrix in matrix sensing. However, these measurements don't satisfy the RIP property.  We will modify the learning algorithm slightly to cope with it. See Section~\ref{sec:quadratic} for details.

\noindent {\bf Organization:} 
The rest of this paper is organized as follows:
In Section \ref{sec:prelim}, we define notations and present a review of the restricted isometry property. 
In Section \ref{sec:rank1}, we lay out the key theoretical insights towards proving Theorem \ref{thm:intro-main} and give the analysis for the rank-1 case as a warm-up. 
In Section \ref{sec:mainproof}, we outline the main steps for proving Theorem \ref{thm:intro-main} and Section~\ref{sec:proofprop} completes the proofs of these steps. Section~\ref{sec:quadratic} and Section~\ref{sec:proofs:q} give the proof of Theorem~\ref{thm:main-quadratic}. Section~\ref{sec:exp} contains numeric simulations. 
Finally, Section \ref{sec:rip} provide the proofs of concentration properties we have used.
\paragraph{Notations:}
Let $\Id_U$ denotes the projection to the column span of $U$, and let $\Id$ denotes the identity matrix.  Let $U^+$ denote the Moore-Penrose pseudo-inverse of the matrix $U$. Let $\Norm{\cdot}$ denotes the Euclidean norm of a vector and spectral norm of a matrix.  Let $\Norm{\cdot}_F$ denote the Frobenius norm of a matrix. 
Suppose $A\in \R^{m\times n}$, then $\sigma_{\max}(A)$ denote its largest singular value and $\sigma_{\min}(A)$ denotes its $\min\{m,n\}$-th largest singular value. Alternatively, we have $\sigma_{\min}(A) = \min_{x:\|x\|=1}\Norm{Ax}$. Let $\inner{A, B} = \trace(A^\top B)$ denote the inner product of two matrices. We use $\sin(A,B)$ to denote the sine of the principal angles between the columns spaces of $A$ and $B$. 

Unless explicitly stated otherwise, $O(\cdot)$-notation hides absolute multiplicative constants.
Concretely, every occurrence of $O(x)$ is a placeholder for some function $f(x)$ that satisfies $\forall x\in \R,\, |{f(x)}|\le C|x|$ for some absolute constant $C>0$. Similarly, $a\lesssim b$ means that there exists an absolute constant $C> 0$ such that $a \lesssim Cb$. We use the notation $\poly(n)$ as an abbreviation for $n^{O(1)}$.  

\section{Preliminaries and Related Work}\label{sec:prelim}
Recall that we assume $\bXg$ is rank-$r$ and positive semidefinite. Let $\bXg = \bUg \bSigmag\bUg^{\top}$ be the eigen-decomposition of $\bXg$, where $\bUg \in \mathbb{R}^{d \times r}$ is an orthonormal matrix and $\bSigmag \in \mathbb{R}^{r \times r}$ is a diagonal matrix. The assumptions that $\|\bXg\| =1$ and $\sigma_{r}(\bXg) = 1/\kappa$ translate to that $\forall i \in [r],  1/\kappa \le \Sigma^\star_{ii}\le 1$. 
Under the above notations, we see that the target solution for the variable $U$ is equal to $U = \bUg {\Sigma^\star}^{1/2}R$ where $R$ can be arbitrary orthonormal matrix. 
For convenience, we define the matrix $M_t$ as
\begin{align}
M_t = \frac{1}{m}\sum_{i = 1}^m\inner{\bA_i, \bU_t\bU_t^\top - \bXg}\bA_i \label{eqn:Mt}
\end{align}
\noindent Then the update rule can be rewritten as 
\begin{align}
U_{t+1} = (\Id- \eta M_t)U_t\label{eqn:def-Ut}
\end{align}
\noindent where $\Id$ is the identity matrix. One of the key challenges is to understand how the matrix $\Id - \eta M_t$ transforms $U_t$, so that $U_0$ converges the target solution $\bUg {\Sigma^\star}^{1/2}R$ quickly.

Suppose that $A_1,\dots, A_m$ are drawn from Gaussian distribution and  optimistically suppose that they are \textit{independent} with $U_t$. Then, we have that $M_t \approx U_tU_t^\top - \bXg, $ since the expectation of $M_t$ with respect to the randomness of $A_i$'s is equal to $U_tU_t^\top - \bXg$. However, they are two fundamental issues with this wishful thinking: a) obviously $U_t$ depends on $A_i$'s heavily for $t> 1$, since in every update step $A_i$'s are used; b) even if $A_i$'s are independently with $U_t$, there are not enough $A_i$'s to guarantee $M_t$ concentrates around its mean $U_tU_t^\top - \bXg$ in Euclidean norm. To have such concentration, we need $m > d^2$, whereas we only have $m = d \times \poly(r \log d)$ samples.

\paragraph{Restricted isometry propety:} The restricted isometry property (RIP) allows us to partially circumvent both the technical issues a) and b) above. It says that using the set of linear measurement matrices $A_1,\dots, A_m$, we can preserve the Frobenius norm of any rank-$r$ matrices approximately.

\begin{defn}\label{def:rip}(Restricted isometry property~\cite{recht2010guaranteed})  A set of linear measurement matrices  $A_1,\dots, A_m$ in $\mathbb{R}^{d\times d}$ satisfies $(r,\delta)$-restricted isometry property (RIP) if for any $d\times d$ matrix $X$ with rank at most $r$, we have
	\begin{align}
	(1-\delta)\norm{X}_F^2 \le \frac{1}{m}\sum_{i = 1}^m \inner{A_i, X}^2 \leq (1+\delta)\norm{X}_F^2 \mper \label{eqn:RIP}
	\end{align}
\end{defn}
\noindent The crucial consequence of RIP that we exploit in this paper is the meta statement as follows: 
\begin{align}
\textup{$\mathcal{M}(Q) := \frac{1}{m}\sum_{i = 1}^m\inner{\bA_i, Q}\bA_i$ behaves like $Q$ for approximately low-rank $Q$} \label{eqn:meta}\end{align}

\noindent We will state several lemmas below that reflect the principle above. The following lemma says that $\langle\mathcal{M}(X), Y\rangle$ behaves like $\langle X, Y\rangle$ for low-rank matrices $X$ and $Y$. 

\begin{lem}\cite[Lemma 2.1]{candes2008RIP}\label{lem:RIP3}
	Let $\{ A_i \}_{i=1}^m$ be a family of matrices in $\Real^{d \times d}$
	that satisfy $(r, \delta)$-restricted isometry property.
	Then for any matrices $X, Y \in \Real^{d \times d}$ with rank at most $r$,
	we have:
	\[ \bigabs{\frac 1 m \sum_{i=1}^m \innerProduct{A_i}{X} \innerProduct{A_i}{Y}    - \innerProduct{X}{Y} }
	\le \delta \normFro{X} \normFro{Y} \]
\end{lem}

\noindent The following lemma says that $\mathcal{M}(X)$ behaves like $X$ when multiplied by a matrix $R$ with small operator norm.  
\begin{lem}\label{lem:property_1}
	Let $\{A_i\}_{i=1}^m$ be a family of matrices in $\Real^{d \times d}$ that
	satisfy $(r, \delta)$-restricted isometry property.
	Then for any matrix $X \in \mathbb{R}^{d \times d}$ of rank at most $r$,
	and any matrix $R \in \mathbb{R}^{d \times d'}$, where $d'$ can be any positive integer,
	we have:
	\[ \bignorm{ \frac{1}{m}\sum_{i = 1}^m \langle \bA_i , \bX \rangle \bA_i  \bR - \bX \bR } \leq  \delta \normFro{X} \cdot \norm{R}. \]
\end{lem}
\noindent Lemma~\ref{lem:property_1} is proved in Section~\ref{sec:rip}\footnote{We suspect that Lemma \ref{lem:property_1} is already known, however we haven't been able to find a reference.}. We can also extend Lemma \ref{lem:property_1} to the cases when $X$ has a higher rank (see Lemma~\ref{lem:RIP4} and Lemma~\ref{lem:property_2}). The bounds are not as strong as above (which is inevitable because we only have $m$ measurements), but are useful when $X$ itself is relatively small.

\subsection{Related Work}

\paragraph{Generalization theory beyond uniform convergence: } 
This work builds upon the remarkable work of Gunasekar et al.~\cite{gunasekar2017implicit}, which raises the conjecture of the implicit regularization in matrix factorization models and provides theoretical evidence for the simplified setting where the measurements matrices are commutable.  Implicit regularization of gradient descent is studied in the logistic regression setting by Soudry et al.~\cite{soudry2017implicit}.

Recently, the work of Hardt et al.~\cite{hardt2015train} studies the implicit regularization provided by stochastic gradient descent through uniform stability~\cite{bousquet2002stability,mukherjee2006learning,shalev2010learnability}. Since the analysis therein is independent of the training labels and therefore it may give pessimistic bounds~\cite{zhang2017learnability}. Brutzkus et al.~\cite{brutzkus2017sgd} use a compression bound to show network-size independent generalization bounds of one-hidden-layer neural networks on linearly separable data. 

Bartlett et al.~\cite{bartlett2017spectrally}, Neyshabur et al.~\cite{neyshabur2017pac}, and Cisse et al.~\cite{cisse2017parseval} recently prove spectrally-normalized margin-based generalization bounds for neural networks. Dziugaite and Roy~\cite{dziugaite2017computing} provide non-vacuous generalization bounds for neural networks from PCA-Bayes bounds. As pointed out by Bartlett et al.~\cite{bartlett2017spectrally}, it's still unclear why SGD without explicit regularization can return a large margin solution. This paper makes progress on explaining the regularization power of gradient descent, though on much simpler non-linear models.

\paragraph{Matrix factorization problems: }
Early works on matrix sensing and matrix factorization problems use convex relaxation  (nuclear norm minimization) approaches and obtain tight sample complexity bounds~\cite{recht2010guaranteed, srebro2005rank,candes2009exact,recht2011simpler,candes2011robust}. Tu et al.~\cite{tu2015low} and Zheng and Lafferty~\cite{zheng2016convergence} analyze the convergence of non-convex optimization algorithms from spectral initialization. The recent work of Ge et al.~\cite{ge2016matrix} and Bhojanapalli et al. \cite{bhojanapalli2016personal}
shows that the non-convex objectives on matrix completion and matrix sensing with low-rank parameterization don't have any spurious local minima, and stochastic gradient descent algorithm on them converges to the global minimum. 
Such a phenomenon was already known for the PCA problem and recently shown for phase retrieval, robust PCA,
and random tensor decomposition as well (e.g., see~\cite{srebro2003weighted, ge2016matrix,bhojanapalli2016personal,ge2017no,ge2017on,sun2016phase} and references therein). 
Soltanolkotabi et al.~\cite{2017arXiv170704926S} analyzes the optimization landscape of over-parameterized one-hidden-layer neural networks with quadratic activations. 
Empirically, Jose et al.~\cite{jose2017kronecker} show that factorized parameterizations of recurrent neural networks provide additional regularization effect. 

\section{Proof Overview and Rank-1 Case}\label{sec:rank1}
\newcommand{\Pustar}{(\Id - \Id_{\vecgt})}

In this section, we demonstrate the key ideas of our proofs and give an analysis of the rank-1 case as a warm-up. 
The main intuition is that the iterate $U_t$ stays approximately low-rank in the sense that:
\begin{enumerate}
	\item[(a)] The $(r+1)$-th singular value $\sigma_{r+1}(U_t)$ \textup{remains small} for any $t\ge 0$;
	\item[(b)] The top $r$ singular vectors and singular values of $U_tU_t^\top$ converge to those of $\bXg$ in logarithmic number of iterations. 
\end{enumerate}
Propositions (a) and (b) can be clearly seen when the number of observations $m$ approaches infinity and $A_1,\dots, A_m$ are Gaussian measurements. Let's define the population risk $\bar{f}$ as 
\begin{align}\label{eq_test_error}
\bar{f}(U_t) = \Exp_{(A_i)_{k\ell} \sim N(0,1)}\left[f(U_t)\right] = \|U_tU_t^\top- \bXg\|_F^2
\end{align}
In this case, the matrix $M_t$ (defined in ~\eqref{eqn:Mt}) corresponds to $U_tU_t^\top - \bXg$, and therefore the update rule for $U_t$ can be simply rewritten as 
\begin{align*}
U_{t+1} & = U_t - \eta \nabla \bar{f}(U_t) = U_t - \eta(U_tU_t^\top - \bXg)U_t \\
& = U_t(\Id - \eta U_t^\top U_t)+ \eta \bXg U_t
\end{align*}
Observe that the term $\eta \bXg U_t$ encourages the column span of $U_{t+1}$ to move towards the column span of $\bXg$, which causes the phenomenon in Proposition (b). On the other hand, the term $U_t(\Id - \eta U_tU_t^\top)$ is performing a contraction of all the singular values of $U_t$, and therefore encourages them to remain small. 
As a result, $U_t$ decreases in those directions that are orthogonal to the span of $\bXg$, because there is no positive force to push up those directions.

So far, we have described intuitively that the iterates of GD on the population risk remains approximately low-rank. 
Recall that the difficulty was that the empirical risk $f$ doesn't uniformly concentrate well around the population risk $\bar{f}$.\footnote{Namely, we don't have uniform convergence results in the sense that $|f(U)-\bar{f}(U)|$ is small for all matrices $U$.  (For examples, for many matrices we can have $f(U) = 0$ but $\bar{f}(U) \gg 0$ because we have more variables than observations.)} However, the uniform convergence can occur, at least to some extent, in the restricted set of approximately low-rank matrices! In other words, since the gradient descent algorithm only searches a limited part of the whole space, we only require restricted uniform convergence theorem such as restricted isometry property. 
Motivated by the observations above, a natural meta proof plan is that:
\begin{enumerate}
	\item[1.] The trajectory of the iterates of GD on the population risk stays in the set of approximately low-rank matrices.
	\item[2.] The trajectory of the empirical risk behaves similarly to that of the population risk in the set of approximately low-rank matrices.
\end{enumerate}
It turns out that implementing the plan above quantitatively can be technically challenging: the distance from the iterate to the set of low-rank matrices can accumulate linearly in the number of steps. Therefore we have to augment the plan with a strong result about the rate of convergence: 
\begin{enumerate}
	\item[3.] The iterates converge to the true solution $\bXg$ fast enough before its effective rank increases.
\end{enumerate}
For the rest of this section, we demonstrate a short proof of the rank-1 case to implement the intuitions described above.
We note that the results of the rank-$r$ case in Section~\ref{sec:mainproof} is significantly tighter than the results presented in this section.
The analysis involves more delicate techniques to control the growth of the top $r$ eigenvectors, and requires a much stronger convergence analysis.

\subsection{Warm-up: Rank-1 Case}

In this subsection, we assume that $\bXg = \vecgt \vecgt^{\top}$ for $\vecgt \in \R^{d\times 1}$ and that $\norm{\vecgt} = 1$. 
We decompose the iterates $U_t$ into the subspace of $\vecgt$ and its complement:
\begin{align}
	U_t & = \Id_{\vecgt}U_t +  \Pustar U_t\nonumber \\
	 & := \vecgt r_t^{\top} + E_t\label{eqn:repre}
\end{align}
where we denote by $r_t:= U_t^{\top} \vecgt$ and  $E_t := \Pustar U_t$.
\footnote{Observe that we have restricted the column subspace of the signal term $R_t = \vecgt r_t^{\top}$, so that $R_tR_t^{\top}$ is always a multiple of $\bXg$.
In section \ref{sec:mainproof}, we will introduce an adaptive subspace instead to decompose $U_t$ into the signal and the error terms.}

In light of the meta proof plan discussed above,  we will show that the spectral norm and Frobenius norm of the ``error term'' $E_t$ remains small throughout the iterations, whereas the ``signal'' term $u^\star r_t^\top$ grows exponentially fast (in the sense that the norm of $r_t$ grows to 1.) Note that any solution with $\|r_t\|=1$ and $E_t =0$ will give exact recovery, and for the purpose of this section we will show that $\|r_t\|$ will converges approximately to 1 and $E_t$ stays small.

Under the representation~\eqref{eqn:repre}, from the original update rule~\eqref{eqn:def-Ut}, we derive the update for $E_t$:
	\begin{align}
E_{t+1} &= (\Id - \Id_{\vecgt}) \cdot (\Id - \eta M_t) U_t \nonumber\\&= E_t - \eta \cdot \Pustar M_t U_t \label{eqn:update-et}
\end{align}

Throughout this section, we assume that $r=1$ and the set of measurement matrices $(A_1,\dots, A_m)$ satisfies $(4,\delta)$-RIP with $\delta \le c$ where $c$ is a sufficiently small absolute constant (e.g., $c=0.01$ suffices). 
\begin{thm}\label{thm:rank1} In the setting of this subsection, suppose $\alpha \le \delta \sqrt{ \frac {1} d \log{\frac 1 {\delta}}}$ and $\eta \lesssim  c\delta^2\log^{-1} (\frac 1 {\delta \alpha})$.
			Then after $T = \Theta(\log{\frac 1 {\alpha \delta}} / \eta)$ iterations, we have:
	\[ \normFro{U_{T} U_{T}^{\top} - \bXg} \lesssim \delta\log{\frac 1 {\delta}} \]
\end{thm}

As we already mentioned, Theorem \ref{thm:rank1} is weaker than Theorem~\ref{thm:intro-main} even for the case with $r=1$. In Theorem~\ref{thm:intro-main} (or Theorem~\ref{thm:technical-main}), the final error depends linearly on the initialization, whereas the error here depends on the RIP parameter. Improving Theorem~\ref{thm:rank1} would involve finer inductions, and we refer the readers to Section~\ref{sec:mainproof} for the stronger results.

The following lemma gives the growth rate of $E_t$ in spectral norm and Euclidean norm, in a typical situation when $E_t$ and $r_t$ are bounded above in Euclidean norm.
\begin{prop}[Error dynamics]\label{prop:rank1_Et}
	In the setting of Theorem~\ref{thm:rank1}. Suppose that $\normFro{E_t} \le 1/2$ and $\norm{r_t}^2 \le 3/2$.
	Then $E_{t+1}$ can be bounded by
	\begin{align}
	\normFro{E_{t+1}}^2 &\le \normFro{E_t}^2 + 2\eta \delta \norm{E_t r_t} + 1.5 \eta \delta \norm{E_t}^2 + 9\eta^2 \label{eqn:etfro}.\\
	\norm{E_{t+1}} & \le (1 + 2\eta\delta) \norm{E_t} + 2\eta\delta \norm{r_t}. \nonumber
	\end{align}
\end{prop}

A recurring technique in this section, as alluded before, is to establish the approximation
\begin{align*}
U_{t+1} = U_t - \eta M_t U_t \approx U_t - \eta (U_tU_t-\bXg)U_t
\end{align*}
As we discussed in Section~\ref{sec:prelim}, if $U_tU_t-\bXg$ is low-rank, then the approximation above can be established by Lemma~\ref{lem:RIP3} and Lemma~\ref{lem:property_1}. However, $U_tU_t-\bXg$ is only approximately low-rank, and we therefore we will decompose it into
\begin{align}
U_t U_t^{\top} - \bXg = \underbrace{(U_t U_t^{\top} - \bXg - E_t E_t^{\top})}_{\textup{rank} \le 4} + \underbrace{E_tE_t^\top}_{\textup{second-order in $E_t$}} \label{eqn:decompose}
\end{align}
Note that $U_t U_t^{\top} - \bXg - E_t E_t^{\top} = \|r_t\|^2 u^\star {u^\star}^\top + u^\star r_t^\top E_t^\top + E_tr_t{u^\star}^\top - \bXg$ has rank at most 4, and therefore we can apply Lemma~\ref{lem:RIP3} and Lemma~\ref{lem:property_1}. For the term $E_tE_t^\top$, we can afford to use other looser bounds (Lemma~\ref{lem:RIP4} and~\ref{lem:property_2}) because $E_t$ itself is small. 
\begin{proof}[Proof Sketch of Proposition~\ref{prop:rank1_Et}]
Using the update rule~\eqref{eqn:update-et} for $E_t$, we have that 
\begin{align}
	\normFro{E_{t+1}}^2
	&= \normFro{E_t}^2 - 2\eta \cdot \innerProduct{E_t}{\Pustar M_t U_t}
	+ \eta^2 \normFro{\Pustar M_t U_t}^2 \label{eq:rank1_Et_fro}
\end{align}
	When $\eta$ is sufficiently small and $\normFro{M_t}, \norm{U_t}$ are bounded from above, the third term on the RHS is negligible compared to the second term. Therefore, we focus on the second term first.
\begin{align}
	&\quad~ \innerProduct{E_t}{\Pustar M_t U_t} \nonumber \\
	&= \frac 1 m \sum_{i=1}^m \innerProduct{A_i}{U_t U_t^{\top} - \bXg} \innerProduct{ A_i}{\Pustar E_t U_t^{\top}} \label{eqn:12}
\end{align}
	where in the last line we rearrange the terms and use the fact that $\Pustar$ is symmetric.
	Now we use Lemma \ref{lem:RIP3} to show that equation~\eqref{eqn:12} is close to 
	$\innerProduct{U_t U_t^{\top} - \bXg}{\Pustar E_t U_t^{\top}}$, which is its expectation w.r.t the randomness of $A_i$'s if $A_i$'s were chosen from spherical Gaussian distribution.
	If $U_t U_t^{\top} - \bXg$ was a rank-1 matrix, then this would follow from Lemma \ref{lem:RIP3} directly. However, $U_tU_t^\top$ is approximately low-rank. Thus, we decompose it into a low-rank part and an error part with small trace norm as in equation~\eqref{eqn:decompose}. Since $U_t U_t^{\top} - \bXg - E_t E_t^{\top}$ has rank at most 4,  we can apply Lemma~\ref{lem:RIP3} to control the effect of $A_i$'s, 
	\begin{align}
& \frac 1 m \sum_{i=1}^m \innerProduct{A_i}{(U_t U_t^{\top} - \bXg - E_t E_t^{\top})} \innerProduct{A_i}{E_t U_t^{\top}} \nonumber\\
& \ge \inner{U_t U_t^{\top} - \bXg - E_t E_t^{\top}, E_tU_t^\top} - 	\delta \norm{U_t U_t^{\top} - \bXg - E_t E_t^{\top}}_F \norm{E_t U_t^{\top}}\nonumber \\
& \ge \inner{U_t U_t^{\top} - \bXg - E_t E_t^{\top}, E_tU_t^\top} - 	1.5\delta \norm{E_t U_t^{\top}}\label{eqn:100}
	\end{align}
	where the last inequality uses that 	$\normFro{U_t U_t^{\top} - \bXg - E_t E_t^{\top}}^2 = (1 - \norm{r_t}^2)^2 + 2\norm{E_t r_t}^2 \le 1 + \norm{E_t}^2 \norm{r_t}^2 \le 11/8$.
	
	For the $E_tE_t^\top$ term in the decomposition~\eqref{eqn:decompose}, we have that
 	\begin{align}
 	\frac 1 m \sum_{i=1}^m \innerProduct{A_i}{E_t E_t^{\top}} \innerProduct{A_i}{E_t U_t^{\top}}
    &\ge \frac 1 m \sum_{i=1}^m \innerProduct{A_i}{E_t E_t^{\top}} \innerProduct{A_i}{E_t r_t {u^{\star}}^{\top}} \tag{clearly $\innerProduct{A_i}{E_t E_t^{\top}}^2 \ge 0$} \\
    &\ge \innerProduct{E_t E_t^{\top}}{E_t r_t {u^{\star}}^{\top}} -\delta \norm{E_t E_t^{\top}}_\star\norm{E_t r_t {u^{\star}}^{\top}}	\tag{by Lemma \ref{lem:RIP4}} \\
    &\ge \innerProduct{E_t E_t^{\top}}{E_t r_t {u^{\star}}^{\top}} - 0.5\delta \norm{E_t r_t}.	\label{eqn:101}
 	\end{align}
  Combining equation~\eqref{eqn:12}, ~\eqref{eqn:100} and~\eqref{eqn:101}, we conclude that 
 	\begin{align} \innerProduct{E_t}{\Pustar M_t U_t} 
    &\ge \innerProduct{U_t U_t^{\top} - \bXg}{E_t U_t^{\top}} - \innerProduct{E_t E_t^{\top}}{E_t E_t^{\top}} - 2\delta \norm{E_t r_t} - 1.5 \delta\norm{E_t}^2, \label{eqn:103}
	\end{align}
	where we have used that $\norm{E_t U_t^{\top}} \le \norm{E_t r_t} + \norm{E_t E_t^{\top}} = \norm{E_t r_t} + \norm{E_t}^2$.
	Note that ${u^\star}^\top E_t = 0$, which implies that $\bXg E_t = 0$ and $U_t^\top E_t = E_t^\top E_t$. Therefore, 
	\begin{align*}
	\innerProduct{U_t U_t^{\top} - \bXg}{E_t U_t^{\top}} &  = \innerProduct{U_t U_t^{\top} }{E_t U_t^{\top}} \\
	& = \innerProduct{U_t^{\top} }{U_T^\top E_t U_t^{\top}} = \innerProduct{U_t^{\top} }{E_T^\top E_t U_t^{\top}} \\
    & = \innerProduct{E_tU_t^{\top} }{ E_t U_t^{\top}} = \norm{E_t U_t^{\top}}_F^2 \ge \normFro{E_t E_t^{\top}}^2 \ge 0,
	\end{align*}
because $\normFro{E_t U_t^{\top}}^2 = \normFro{E_t E_t^{\top}}^2 + \normFro{E_t r_t {u^{\star}}^{\top}}^2$.
We can also control the third term in RHS of equation~\eqref{eq:rank1_Et_fro} by 
$
\eta^2 \normFro{\Pustar M_t U_t}^2 \le 9\eta^2
$.
Since the bound here is less important (because one can always choose small enough $\eta$ to make this term dominated by the first order term), we left the details to the reader. Combining the equation above with~\eqref{eqn:103} and ~\eqref{eq:rank1_Et_fro}, we conclude the proof of equation~\eqref{eqn:etfro}. 
	Towards bounding the spectral norm of $E_{t+1}$, we use a similar technique to control the difference between $\Pustar M_t U_t$ and $\Pustar (U_t U_t^{\top} - \bXg) U_t$ in spectral norm. By decomposing $U_t U_t^{\top} - \bXg$ as in equation~\eqref{eqn:decompose} and applying Lemma \ref{lem:property_1} 
	and Lemma \ref{lem:property_2} respectively, we obtain that 
	\begin{align*}
\norm{\Pustar M_t U_t - \Pustar (U_t U_t^{\top} - \bXg) U_t} & \le 
		4\delta\left(\normFro{U_t U_t^{\top} - \bXg - E_t E_t^{\top}} + \norm{E_tE_t^\top}_\star \right) \norm{U_t^{\top}}\nonumber\\
		& \le 8 \delta \norm{U_t} \le 8\delta (\norm{r_t} + \norm{E_t}) \tag{by the assumptions that $\norm{E_t}_F\le 1/2, \norm{r_t}\le 3/2$}
	\end{align*}
	Observing that $\Pustar (U_t U_t^{\top} - \bXg) U_t = E_tU_t^\top U_t$. Plugging the equation above into equation~\eqref{eqn:update-et}, we conclude that
	\begin{align*}
		\norm{E_{t+1}} &\le \norm{E_t (\Id - \eta U_t^{\top}U_t)} + 2\eta\delta (\norm{r_t} + \norm{E_t}) \\
		&\le \norm{E_t} + 2\eta\delta (\norm{r_t} + \norm{E_t})
	\end{align*}
\end{proof}
The next Proposition shows that the signal term grows very fast, when the signal itself is not close to norm 1 and the error term $E_t$ is small. 
\begin{prop}[Signal dynamics]\label{prop:rank1_Rt}
In the same setting of Proposition~\ref{prop:rank1_Et}, we have, 
	\begin{align}\label{eq:rank1_Rt_dynamic}
		\bignorm{r_{t+1} - (1 + \eta(1 - \norm{r_t}^2)) r_t} \le
		\eta \norm{E_t}^2 \norm{r_t} +  2\eta\delta(\norm{E_t} + \norm{r_t}).
	\end{align}
\end{prop}

\begin{proof}
By the update rule~\eqref{eqn:def-Ut}, we have that 
	\begin{align*}
		r_{t+1} &= U_{t+1}^{\top} \vecgt = U_t^{\top} (\Id -\eta M_t^{\top}) \vecgt \nonumber\\
			&= r_t - \eta U_t^{\top} M_t^{\top} \vecgt.
				\end{align*}
	By decomposing $M_t = \frac 1 m \sum_{i=1}^m \innerProduct{A_i}{U_t U_t^{\top} - \bXg} A_i$ as in equation \eqref{eqn:decompose},
	and then Lemma \ref{lem:property_1} and \ref{lem:property_2}, we obtain that
	\begin{align*}
		\bignorm{r_{t+1} - (r_t - \eta U_t^{\top} (U_t U_t^{\top} - \bXg) \vecgt)}
		\le \delta \left(\bignormFro{U_t U_t^{\top} - \bXg - E_t E_t^{\top}} + \normNuclear{E_t E_t^{\top}} \right) \norm{U_t}.
	\end{align*}
	Observe that $U_t^{\top} (U_t U_t^{\top} - \bXg) \vecgt = U_t^{\top} U_t r_t - r_t = (r_t r_t^{\top} + E_t^{\top} E_t) r_t - r_t
	= (\norm{r_t}^2 - 1) r_t - E_t^{\top} E_t r_t$.
	Also note that
	$\normFro{U_t U_t^{\top} - \bXg - E_t E_t^{\top}}^2 \le 11/8$ and $\normFro{E_t}^2 \le 1/4$, we have that
	\begin{align*}
		\norm{r_{t+1} - (1 + \eta(1 - \norm{r_t}^2)) r_t}
		\le \eta \norm{E_t^{\top}E_t r_t} +
		2\eta\delta \norm{U_t}
	\end{align*}
	Since $\norm{U_t} \le \norm{r_t} + \norm{E_t}$,
	we obtain the conclusion.
\end{proof}
The following proposition shows that $\|r_t\|$ converges to 1 approximately and $E_t$ remains small by inductively using the two propositions above. 
\begin{prop}[Control $r_t$ and $E_t$ by induction]\label{prop:induction}
In the setting of Theorem~\ref{thm:rank1}, after $T = \Theta(\log(\frac 1 {\alpha\delta})) / \eta)$ iterations,
	\begin{align}
		& \norm{r_T} = 1\pm O(\delta) \\
		& \normFro{E_T}^2 \lesssim \delta^2\log(1/\delta) 	\end{align}
\end{prop}

\begin{proof}[Proof Sketch]
	We will analyze the dynamics of $r_t$ and $E_t$ in two stages.
	The first stage consists of all the steps such that $\Norm{r_t}\le 1/2$,
	and the second stage consists of the rest of steps.  We will show that
	\begin{enumerate}
		\item[a)] Stage 1 has at most $O(\log(\frac 1 {\alpha})/\eta)$ steps.
		Throughout Stage 1, we have
		\begin{align}
			&	\|E_t\|\le 9\delta \label{eqn:er} \\
			&	\norm{r_{t+1}}\ge (1+\eta/3)\norm{r_t} \label{eqn:rr}
		\end{align}
		\item[b)] In Stage 2, we have that 
		\begin{align*}
			& \Norm{E_t}^2_F \lesssim \delta^2\log{\frac 1{\delta}} \\
			& \Norm{r_t}\le 1+O(\delta).
		\end{align*}
		And after at most $O(\log(\frac 1 {\delta})/\eta)$ steps in Stage 2, we have $\Norm{r_t}\ge 1-O(\delta)$. 
	\end{enumerate}

	We use induction with Proposition~\ref{prop:rank1_Et} and ~\ref{prop:rank1_Rt}. For $t=0$, we have that $\|E_0\| = \norm{r_0} = \alpha$ because $U_0 = \alpha B$, where $B$ is an orthonormal matrix. Suppose equation~\eqref{eqn:er} is true for some $t$, then we can prove equation~\eqref{eqn:rr} holds by Proposition~\ref{prop:rank1_Rt}:
	\begin{align*}
		\bignorm{r_{t+1}} & \ge   (1 + \eta (1 - \norm{r_t}^2 - 2\delta - O(\delta^2)))\norm{ r_t} \\
		& \ge (1+\eta/3)\norm{r_t}	\tag{by $\delta \le 0.01$ and $\Norm{r_t}\le 1/2$}
	\end{align*}
	Suppose equation~\eqref{eqn:rr} holds, we can prove equation~\eqref{eqn:er} holds by Proposition~\ref{prop:rank1_Et}.
	We first observe that $t \le O(\log{\frac 1 {\alpha}} / \eta)$,
	since $r_t$ grows by a rate of $1 + \frac{\eta} 3$.
	Denote by $\lambda = 1 + 2\eta\delta$. We have
	\begin{align*}
		& \norm{E_{t+1}} \le \lambda \norm{E_t} + 2\eta\delta \norm{r_t}
		\Rightarrow \frac {\norm{E_{t+1}}} {{\lambda}^{t+1}} \le \frac{\norm{E_t}}{{\lambda}^t} + 2\eta\delta \times \left(\frac{\norm{r_t}} {{\lambda}^{t+1}} \right) \\
		\Rightarrow & \norm{E_{t+1}} \le {\lambda}^{t+1} \times \alpha +
		2\eta\delta \times \sum_{i=0}^t \frac{\lambda^i \norm{r_t}}{(1+\frac{\eta}3)^i} \tag{by $\norm{r_{t+1}} \ge \norm{r_t} (1 + \frac {\eta} 3)$.}
		\le 9\delta,
	\end{align*}
	where the last inequality uses that
	\begin{align*}
		& \lambda^{t+1} \alpha \le \alpha \times \exp(2\eta\delta \times O(\log(\frac 1 {\alpha})) / \eta) = \alpha^{1 - O(\delta)} \le o(\delta), \mbox{ and} \\
		& \sum_{i=1}^t \frac{\lambda^i \norm{r_t}}{(1+\frac{\eta}3)^i} \le \frac{1 + \frac{\eta}3}{2 (\frac{\eta}3 - 2\eta\delta)}
	\end{align*}

For claim b), we first apply the bound obtained in claim a) on $\norm{E_t}$ and the fact that $\norm{U_t} \le 3/2$
(this follows trivially from our induction, so we omit the details).
We have already proved that $\norm{E_t} \le 9\delta$ in the first stage.
For the second stage, as long as the number of iterations is less than $O(\log(\frac 1 {\delta}) / \eta)$,
we can still obtain via Proposition \ref{prop:rank1_Et} that:
\[ \norm{E_t} \le 9\delta + 4\eta\delta \times O(\log(\frac 1 {\delta} / \eta))
= O(\delta \log(\frac 1 {\delta})) \tag{since $\norm{E_t} \le 1/2$ and $\norm{r_t} \le 3/2$} \]
To summarize, we bound $\normFro{E_t}$ as follows:
	\begin{align*}
		\normFro{E_{t}}^2 & \le \normFro{E_0}^2 + 9\eta^2 t + O(\eta t \delta^2 \log(\frac 1 {\delta})) \lesssim \delta^2\log(\frac 1{\delta}),
	\end{align*}
	because $\normFro{E_0}^2 = \alpha^2 d \le \delta^2 \log{\frac 1 {\delta}}$,
	$\eta t \le O(1)$,
	and	$\eta^2 t \le O(\eta) \le O(\delta^2 \log(\frac 1{\delta}))$.
		For the bound on $\norm{r_t}$, we note that since $\norm{E_t} \le O(\delta \log(\frac 1 {\delta}))$, we can simplify Proposition~\ref{prop:rank1_Rt} by 
\begin{align*}
	\bignorm{r_{t+1} -  (1 + \eta (1 - \norm{r_t}^2))r_t}
		\le 2\eta\delta\norm{r_t} + O(\eta\delta^2\log^2(\frac 1 {\delta})).
\end{align*}
	The proof that $\norm{r_t} = 1 \pm O(\delta)$ after at most $O(\log(\frac 1{\delta}) / \eta)$ steps follows similarly by induction.
	The details are left to the readers.
\end{proof}
Theorem~\ref{thm:rank1} follows from Proposition~\ref{prop:induction} straightforwardly. 
\begin{proof}[Proof of Theorem~\ref{thm:rank1}]
Using the conclusions of Proposition~\ref{prop:induction}, we have
	\begin{align*}
	\bignormFro{U_T U_T^{\top} - \bXg}^2
	&= (1 - \norm{r_T}^2)^2 + 2 \norm{E_T r_T}^2 + \normFro{E_T E_T^{\top}}^2 \\
	&\le (1 - \norm{r_T}^2)^2 + 2 \norm{E_T}^2 \norm{r_T}^2 + \normFro{E_T}^4 \\
	&\le O(\delta^2) + O(\delta^2\log^2{\frac 1{\delta}}) +
	O(\delta^4 \log^2(\frac 1 {\delta}))
	= O(\delta^2\log^2{\frac 1 {\delta}})
	\end{align*}
\end{proof}

\section{Proof Outline of Rank-$r$ Case}\label{sec:mainproof}

In this section we outline the proof of Theorem~\ref{thm:intro-main}. The proof is significantly more sophisticated than the rank-1 case (Theorem ~\ref{thm:rank1}), because the top $r$ eigenvalues of the iterates grow at different speed.
Hence we need to align the signal and error term in the right way so that the signal term grows monotonically.
Concretely, we will decompose the iterates into a signal and an error term according to a dynamic subspace, as we outline below. Moreover, the generalization error analysis here is also tighter than Theorem~\ref{thm:rank1}.  
We first state a slightly stronger version of Theorem~\ref{thm:intro-main}:

\begin{thm}\label{thm:technical-main}
	There exists a sufficiently small absolute constant $c > 0$ such that the following is true.
	For every $\alpha \in (0, c/d)$, assume that the set of measurement matrices $(A_1,\dots, A_m)$ satisfies $(r,\delta)$-restricted isometry property with $\delta \le c/(\kappa^3\sqrt{r} \log^2 \frac {d}{\alpha})$, $\eta \le c\delta$, and $T_0 =  \max\left\{\frac{\kappa\log(d/\alpha)}{\eta}, \frac 1 {\eta \sqrt{d\kappa\alpha}}\right\}$.
	For every $t\lesssim T_0$,
	$$\Norm{U_tU_t^\top-\bXg}_F^2 \leq (1 - \eta /(8 \kappa))^{t - T_0} +O(\alpha \sqrt d / \kappa^2).$$
	As a consequence, for $T_1 =\Theta((\kappa\log (\frac d {\alpha}))/\eta)$, we already have
	\[ \Norm{U_{T_1}U_{T_1}^\top-\bXg}_F^2 \lesssim \alpha \sqrt d / \kappa^2. \]
\end{thm}

When the condition number $\kappa$ and rank $r$ are both constant,  this theorem says that if we shoot for a final error $\varepsilon$, then we should pick our initialization $\bU_0 = \alpha B$  with $\alpha = O(\varepsilon/d)$. As long as the RIP-parameter $\delta = O(\frac{1}{\log \frac{d}{\veps}})$, after $O( \log \frac{d}{\veps})$ iterations we will have that $\Norm{U_tU_t^\top-\bXg}_F^2 \leq \veps$.

Towards proving the theorem above, we suppose the eigen-decomposition of $\bXg$ can be written as $\bXg = \bUg \bSigmag\bUg^{\top}$ where $\bUg \in \mathbb{R}^{d \times r}$ is an orthonormal matrix $\bSigmag \in \mathbb{R}^{r \times r}$ is a diagonal matrix. We maintain the following decomposition of $U_t$ throughout the iterations: 
\begin{align}
U_t = \underbrace{\Id_{S_t}U_t}_{:=Z_t} + \underbrace{(\Id - \Id_{S_t}) U_t}_{:=E_t}\label{eqn:def-ze}
\end{align}
Here $S_t$ is $r$ dimensional subspace that is recursively defined by 
\begin{align}
S_0 &= \sp(\bUg)  \\
S_t &= (\Id-\eta M_{t-1})\cdot S_{t-1},~\forall\, t \ge 1.
\end{align}
Here $(\Id-\eta M_t)\cdot S_{t-1}$ denotes the subspace $\{(\Id-\eta M_t)v : v\in S_{t-1}\}$. Note that $\rank(S_0)= \rank(\bUg) = r$, and thus by induction we will have that for every $t \ge 0$, 
\begin{align*}
& \sp(Z_t)\subset \sp(S_t), \\
&\rank(Z_t)  \le \rank(S_t) \le r.
\end{align*}
Note that by comparison, in the analysis of rank-1 case, the subspace $S_t$ is chosen to be $\sp(\bUg)$ for every $t$, but here it starts off as $\sp(\bUg)$ but changes throughout the iterations. We will show that $S_t$ stays close to $\sp(\bUg)$. Moreover, we will show that the error term $E_t$ --- though growing exponentially fast --- always remains much smaller than the signal term $Z_t$, which grows exponentially with a faster rate. Recall that $\sin(A,B)$ denotes the principal angles between the column span of matrices $A,B$. We summarize the intuitions above in the following theorem. 
\begin{thm}\label{thm:induction}	
	There exists a sufficiently small absolute constant $c > 0$ such that the following is true: For every $\alpha \in (0, c/d)$, assume that the set of measurement matrices $(A_1,\dots, A_m)$ satisfies $(r,\delta)$-RIP with $\delta \le c^4/(\kappa^3\sqrt{r} \log^2 \frac{d}{\alpha})$. $\rho = O(\frac{\sqrt{r}\delta}{\kappa\log(\frac d{\alpha})})$, $\eta \le c\delta$.
	Then for $t \le T_1 = \Theta({\frac{\kappa}{\eta} \log(\frac d {\alpha})})$, we have that
	\begin{align}
	& \sin(Z_t, \bUg) \lesssim \eta \rho t \label{eqn:201}\\
	& \|E_t\|\le (1+ O(\eta^2 \rho t))^t\|E_0\| \le 4\|E_0\|\le 1/d\label{eqn:202}\\
	& \sigma_{\min}({U^\star}^\top Z_t)\ge \|E_t\| \label{eqn:203}\\
	&	\sigma_{\min}(\bUg^\top Z_{t})\geq \min\left\{\left(1+\frac{\eta}{8\kappa}\right)^t\sigma_{\min}(\bUg^\top  Z_0), \frac{1}{2\sqrt{\kappa}} \right\}\label{eqn:205}\\
	& \|Z_{t}\|\le 5\label{eqn:204}
	\end{align}
	It follows from equation \eqref{eqn:205} that after $\Theta(\frac{\kappa}{\eta}\log(\frac d{\alpha}))$ steps, we have
	$\sigma_{\min}(Z_{t}) \ge \frac 1 {2\sqrt{\kappa}}.$
\end{thm}
Note that the theorem above only shows that the least singular value of $Z_t$ goes above $1/(2\sqrt{\kappa})$. The following proposition completes the story by showing that once the signal is large enough, $U_tU_t^\top$ converges with a linear rate to the desired solution $\bXg$ (up to some small error.)

\begin{prop}\label{prop:convergence}
	In the setting of Theorem~\ref{thm:induction}, suppose $\| Z_t \| \leq 5$, $\sin(Z_{t}, \bUg) \leq 1/3$, and $\sigma_{\min}(Z_t) \geq \frac{1}{2\sqrt{\kappa}}$, then we have:
	\begin{align*}
		\Norm{U_{t+1}U_{t+1}^\top- \bXg}_F^2 \le (1 - \frac {\eta} {8 \kappa}) \Norm{U_{t}U_{t}^\top-\bXg}_F^2   +  O\left(\eta \sqrt{dr} \Norm{E_t} \right).
	\end{align*}
\end{prop}

\noindent We defer the proof of Proposition~\ref{prop:convergence} to Section~\ref{sec:prop:convergence}, which leverages the fact that function $f$ satisfies the Polyak-Lojasiewicz condition~\cite{polyak1963gradient} when $U_tU_t^\top$ is well-conditioned. The rest of the section is dedicated to the proof outline of Theorem~\ref{thm:induction}. We decompose it into the following propositions, from which Theorem~\ref{thm:induction} follows by induction. 

The following proposition gives the base case for the induction, which straightforwardly follows from the definition $\bU_0 = \alpha B$ where $B$ is an orthonormal matrix. 
\begin{prop}[Base Case]\label{prop:base} 
	In the setting of Theorem~\ref{thm:induction}, we have 
	\begin{align*}
	\Norm{E_0} = \alpha =  \sigma_{\min}(Z_0) \le 1/d,\quad
	\bUg^{\top} E_0 = 0,
	\quad \text{and} \quad  \sin(Z_0, \bUg)   = 0.
	\end{align*}
\end{prop}
The following Proposition bounds the growth rate of the spectral norm of the error $E_t$. 
\begin{prop}\label{prop:fullrankerror}
	In the setting of Theorem~\ref{thm:induction},  suppose $\Norm{E_t}\le 1/d$. Then, 
	\begin{align*}
		\Norm{E_{t+1}} \le \left(1+ \eta O\left(\delta \sqrt{r} + \sin(Z_{t}, \bUg)  \right) \right)\times \Norm{E_t}.
	\end{align*}
	When $\normFro{U_tU_t^{\top} - \bXg}$ is small, the growth of $E_t$
	becomes slower:
	\[ \Norm{E_{t+1}} \le  \left(1 + \eta O\left(\norm{U_tU_t^\top - \bXg}_F + \sqrt{r}\Norm{E_t}\right)\right) \times  \Norm{E_t}. \]
\end{prop}
The following Proposition shows that effectively we can almost pretend that $Z_{t+1}$ is obtained from applying one gradient descent step to $Z_t$, up to some some error terms. 
\begin{prop}\label{prop:unknown}
	In the setting of Theorem~\ref{thm:induction},  suppose for some $t$ we have 
	$\Norm{E_t} \le 1/d$ and $\Norm{Z_t} \leq 5$.	Then, 
	\begin{align}
		\Norm{Z_{t+1} - (Z_t - \eta \cG(Z_t) - \eta E_t Z_t^{\top} Z_t -2 \eta \Id_{S_t} M_t E_t )} 	\le \eta \tau_t, \label{eqn:reduction}
	\end{align}
	where $\tau_t \lesssim \delta \sqrt{r} \Norm{E_t}$.
\end{prop}
The following proposition shows that the angle between the span of $Z_t$ to the span of $\bUg$ is growing at mostly linearly in the number of steps. 
\begin{prop}\label{prop:low-rank-relative-error}
	In the setting of Theorem~\ref{thm:induction}, assuming equation~\eqref{eqn:reduction} holds for  some $t$ with $\tau_t \le \rho \sigma_{r}(U_t)$, $\Norm{Z_t} \leq 5$ and $\sigma_{\min}(Z_{t}) \geq \Norm{E_t} / 2$. Then, as long as $\sin(Z_{t}, \bUg)  \le \sqrt{\rho}$ we have:
	\begin{align}
	\sin(Z_{t+1}, \bUg) \leq \sin(Z_{t}, \bUg) + O(\eta \rho + \eta \Norm{E_t})  ~\text{ and }~ \Norm{Z_{t + 1}} \leq 5. \label{eqn:Zt}
	\end{align}
\end{prop}
Then we show that the projection of the signal term $Z_t$ to the subspace of $\bUg$ increases at an exponential rate (until it goes above $1/(2\sqrt{k})$). Note that $\bUg$ is sufficiently close to the span of $Z_t$ and therefore it implies that the least singular value of $Z_t$ also grows.

\begin{prop}\label{prop:convergence-and-eigen-grow}
	In the setting of Theorem~\ref{thm:induction}, suppose equation~\eqref{eqn:Zt} holds for some $t$, and $\Norm{Z_t} \leq 5$ and $\sigma_{\min}(Z_{t}) \geq \Norm{E_t} / 2$, we have that:
	\begin{align*}
		\sigma_{\min}(\bUg^\top Z_{t + 1})\geq \min\{(1+\frac{\eta}{8\kappa})\sigma_{\min}(\bUg^\top  Z_t), \frac 1 {2\sqrt{\kappa}}\}.
	\end{align*}
\end{prop}
\noindent The proofs of the above propositions are deferred to Section~\ref{sec:proofprop}. Now we are ready to prove Theorem~\ref{thm:induction}. 
\begin{proof}[Proof of Theorem~\ref{thm:induction}]
	When $t = 0$, the base case follows from Proposition \ref{prop:base}.
	Assume that equations~\eqref{eqn:201}, ~\eqref{eqn:202}, ~\eqref{eqn:203},~\eqref{eqn:205}, and~\eqref{eqn:204} are true before or at step $t$,
	we prove the conclusion for step $t+1$.
	By Proposition~\ref{prop:unknown} we have that equation~\ref{eqn:reduction} are true with $\tau_t \lesssim \delta \sqrt{r}$.
	We have that $\sigma_r(U_t) \ge \sigma_r(Z_t)$, because the column subspace of $Z_t$ and $E_t$ are orthogonal.
	By induction hypothesis, $\sigma_r(Z_t) \ge \norm{E_t} / 2$, and $\sigma_r(U_t) \ge \norm{E_t} / 2$.
	Set the $\rho$ in Proposition~\ref{prop:low-rank-relative-error} to $O(\frac {\delta\sqrt{r}}{\kappa\log(\frac d {\alpha})})$.
	When $t \leq T_1 = O(\frac{\kappa}{\rho} \log(\frac d {\alpha}))$,
	we have that $\sin(Z_t, \bUg) \leq \eta\rho t \le \sqrt{\rho}$.
	Hence the assumptions of Proposition \ref{prop:low-rank-relative-error} are verified.	Therefore,
	\begin{align*}
		\sin(Z_{t+1}, \bUg) \lesssim \eta \rho t + O(\eta \rho) +  O(\eta \norm{E_t}) \lesssim \eta\rho(t+1)
	\end{align*}
	because $\norm{E_t} \le 1/d \ll \rho$, and $\norm{Z_t} \le 5$.
	Next,
	\begin{align}
		\Norm{E_{t+1}} & \le \left(1 + \eta O\left(\delta \sqrt{r} + \sin(Z_{t}, \bUg)  \right) \right)\cdot \Norm{E_t}\nonumber\\
		& \le (1+O(\eta^2\rho T_1))^{T_1} \|E_0\| \tag{by equations ~\eqref{eqn:201}, ~\eqref{eqn:202}, and $\delta\sqrt r \le O(\eta\rho T_1)$}\\
		& \le 4\|E_0\| \le 1/d \tag{since $\eta^2\rho T_1^2 \le O(\delta \sqrt r \kappa \log(\frac d {\alpha})) \le O(1)$}
	\end{align}
	Therefore we can apply Proposition~\ref{prop:convergence-and-eigen-grow}
	to obtain that $\sigma_{\min}(\bUg^{\top} Z_t)$ grows by a rate of at least $1 + \frac {\eta}{8\kappa}$.
	On the other hand by Proposition \ref{prop:fullrankerror},
	$\norm{E_t}$ grows by a rate of at most $1 + \eta O(\delta \sqrt r + \sqrt {\rho}) \le 1 + \frac{\eta}{8\kappa}$.
	Hence we obtain Equation \ref{eqn:203}.
\end{proof}

\noindent Finally we prove Theorem~\ref{thm:technical-main}. 
\begin{proof}
		First of all, for $t \leq T_1 = \Theta(\kappa \log \frac{d}{\alpha} / \eta)$, using Theorem~\ref{thm:induction}, we know that the requirements in  Proposition~\ref{prop:convergence} is satisfied. Applying Proposition~\ref{prop:convergence}, we prove the theorem for $t \leq T_1$. 

Then, we inductively show that the error is bounded by $O(\alpha \sqrt d / \kappa^2)$ from the $T_1$-th iteration  until the $T_0$-th iteration.
Suppose at iteration $t$, we have $\| \bU_t \bU_t^{\top} - \bXg \|_F^2 \lesssim \alpha \sqrt d / \kappa^2$.
Thus, using Proposition~\ref{prop:fullrankerror}, we know that $\norm{E_t}$ grows by a rate of at most $1 + \eta O(\sqrt{\alpha} d^{1/4} / \kappa)$ for this $t $.
This implies that for every $t \in [T_1, T_0]$, we have
\begin{align*}
	\Norm{ E_{t  + 1} }  & \leq \left( 1 + \eta O(\sqrt{\alpha} d^{1/4} / \kappa) \right) \Norm{E_t} \\
	& \leq \Norm{E_{T_1}} \left(1 + \eta O(\sqrt{\alpha} d^{1/4} /\kappa)\right)^{T_0} \leq 4 \Norm{E_{T_1}} \leq 16 \Norm{E_{0}}.
\end{align*}
By inductive hypothesis we recall $\| \bU_t \bU_t^{\top} - \bXg \|_F^2 \lesssim \kappa \sqrt d / \kappa^2$, which implies by elementary calculation that $\sigma_{\min}(\bZ_t) \geq \frac{1}{2 \sqrt{\kappa}}, \Norm{Z_t } \leq 5$ and $\sin(Z_t, \bUg) \leq \frac{1}{3}$ (by using $\bU_t = \bZ_t + \bE_t$ and $\Norm{E_t} \leq 16 \Norm{E_{0}}$). Thus, the requirements in Proposition~\ref{prop:convergence} hold, and applying it we obtain that $\Norm{U_{t+1}U_{t+1}^\top- \bXg}_F^2\lesssim \alpha \sqrt d / \kappa^2$. This completes the induction. 
\end{proof}

 \section{Neural networks with Quadratic Activations}\label{sec:quadratic}
 
In this section, we state the algorithms and generalization bounds for learning over-parameterized neural nets with quadratic activations, and give the key lemma for the analysis. 

Let $(x_1, y_1),\dots, (x_m, y_m)$ be $n$ examples where $x_i$'s are from distribution $\mathcal{N}(0, \Id_{d\times d})$ and $y_i$'s are generated according to equation~\eqref{eqn:qnn}. 
Let  $\hat{y} = \mathbf{1}^\top q(U^\top x)$
be our prediction. We use mean squared loss as the empirical loss. 
For technical reasons, we will optimize a truncated version of the empirical risk as
\begin{align*}
 \tilde{f}(U) = \frac{1}{m}\sum_{i=1}^{n} (\hat{y}_i-y_i)^2\mathbf{1}_{\|U^\top x\|^2\le R}
\end{align*}
for some parameter $R$ that will be logarithmic in dimension later.  We design a variant of gradient descent as stated in Algorithm~\ref{alg:aqnn}.
We remark mostly driven by the analysis, our algorithm has an explicit re-scaling step. It resembles the technique of weight decay~\cite{krogh1992simple}, which has similar effect to that of an $\ell_2$ regularization. In the noiseless setting, the issue with vanilla weight decay or $\ell_2$ regularizer is that the recovery guarantees will depend on the strength of the regularizer and thus cannot achieve zero. An alternative is to use a truncated $\ell_2$ regularizer  that only penalizes when $U_t$ has norm bigger than a threshold. 
Our scaling that we are using is dynamically decided, and in contrast to this truncated regularizer, it scales down the iterate when the norm of $U_t$ is small, and it scales up the iterate when $U_t$ has norm bigger than the norm of $\bUg$.\footnote{But note that such scaling up is unlikely to occur because the iterate stays low-rank} Analyzing standard gradient descent is left for future work. 

 We note that one caveat here is that for technical reason, we assume that we know the Frobenius norm of the true parameter $\bUg$.  It can be estimated by taking the average of the prediction $\frac{1}{m}\sum_{i  = 1}^m y_i$ since $\Exp[y]  = \norm{\bUg}_{F}^2$, and the algorithm is likely to be robust to the estimation error of $\norm{\bUg}_{F}^2$. However, for simplicity, we leave such a robustness analysis for future work. 
\begin{algorithm}\caption{Algorithm for neural networks with quadratic activations}\label{alg:aqnn}
{\bf Inputs: } $n$ examples $(x_1, y_1),\dots, (x_m, y_m)$ where $x_i$'s are from distribution $\mathcal{N}(0, \Id_{d\times d})$ and $y_i$'s are generated according to equation~\eqref{eqn:qnn}. Let $\tau = \norm{\bUg}_{F}^2$. 

{\bf Initialize} $U_0$ as in equation~\eqref{eqn:init}	\\
For $t =1$ to $T$:
\begin{align}
\tilde{U}_t & = U_t - \eta \nabla \tilde{f}(U_t) \nonumber\\
U_{t + 1 } &= \frac{1}{1 - \eta (\| \bU_t\|_{F}^2- \tau)} \tilde{U}_t \nonumber
\end{align}
\end{algorithm}

As alluded before, one-hidden-layer neural nets with quadratic activation closely connects to matrix sensing because we can treat write the neural network prediction by: 
\begin{align*}
{\bf 1}^\top  q({\bUg}^\top x) = \inner{xx^\top, \bUg \bUg^\top}
\end{align*}
Therefore, the $i$-th example $x_i$ corresponds to  the $i$-th measurement  matrix in the matrix sensing via $A_i = x_i{x_i}^\top$.  Assume $\{ x_1{x_1}^\top, \dots,  x_mx_m^\top\}$ satisfies RIP, then we can re-use all the proofs for matrix sensing. However, this set of rank-1 measurement matrices \textit{doesn't} satisfy RIP with high probability. The key observation is that if we truncated the observations properly, then we can make the truncated set of these rank-1 measurements  satisfy RIP property again. Mathematically, we prove the following Lemma.

\begin{lem}\label{lem:gaussian_concentration}
Let $(A_1,\dots, A_m) = \{ x_1{x_1}^\top, \dots,  x_mx_m^\top\}$ where $x_i$'s are  i.i.d. from $\sim \mathcal{N}(0, \Id)$. Let  $R = \log\left(\frac{1}{ \delta}\right)$. Then, for every $q, \delta \in [ 0, 0.01]$ and $m \gtrsim d   \log^4 \frac{d}{q \delta }/\delta^2$, we have:  with probability at least $1 - q$, for every symmetric matrix $X$,
\begin{align*}
\Norm{\frac{1}{m} \sum_{i = 1}^m \langle A_i, X \rangle A_i 1_{|\langle A_i, X \rangle| \leq R} - 2 X - \trace(X) \Id} \leq \delta \| \bX \|_{\star}
\end{align*}
\end{lem}
\noindent Suppose $X$ has rank at most  $r$ matrices and spectral norm at most $1$, we have, 
\begin{align}
\Norm{\frac{1}{m} \sum_{i = 1}^m \langle A_i, X \rangle A_i 1_{|\langle A_i, X \rangle| \leq R} - 2 X - \trace(X) \Id} \leq r\delta\label{eqn:truncated_rip}
\end{align}

We can see that equation~\eqref{eqn:truncated_rip} implies Lemma~\ref{lem:property_1} with a simple change of parameters (by setting $\delta$ to be a factor of $r$ smaller). The proof uses standard technique from supreme of random process, and is deferred to Section~\ref{sec:proofs:q}. 

Theorem~\ref{thm:main-quadratic} follows straightforwardly from replacing the RIP property by the Lemma above. We provide a proof sketch below. 

\begin{proof}[Proof Sketch of Theorem~\ref{thm:main-quadratic}]
The basic ideas is to re-use the proof of Theorem~\ref{thm:technical-main} at every iteration.  First of all, we will replace all the RIP properties \footnote{We only require RIP on symmetric matrices.} in \eqref{place_1}, \eqref{eqn:6}, \eqref{eq_MtZt}, \eqref{eq:place_8}, \eqref{place_2}, \eqref{place_4}, \eqref{eq:y_1} and \eqref{eq_zt_signal}  by Lemma~\ref{lem:gaussian_concentration}.  The only difference is that we will let the $\delta$ when applying Lemma~\ref{lem:gaussian_concentration} to be $1/r$ smaller than the $\delta$ in Theorem~\ref{thm:technical-main}.

We note that in Lemma~\ref{lem:gaussian_concentration}, there is an additional scaling of the identity term compared to Lemma~\ref{lem:property_1}. This is the reason why we have to change our update rule.
We note that the update rule in Algorithm~\ref{alg:aqnn} undo the effect of this identity term and is identical to the update rule for matrix sensing problem:
Let $c_t = \trace(U_t U_t^{\top}) - \trace(\bXg)$.
Denote by $M_t' = M_t - c_t \Id$.
The update in Algorithm~\ref{alg:aqnn} can be re-written as:
\begin{align}
\tilde{\bU}_{t + 1} &= (\Id - \eta M'_t - \eta c_t \Id )\bU_t
\\
&= (1 - \eta c_t) \left(\Id - \frac{\eta}{1 - \eta c_t} M'_t \right) \bU_t\nonumber
\end{align}
Hence we still have $\bU_{t + 1} = \frac{1}{1 - \eta c_t} \tilde{\bU}_{t + 1} = \left(\Id - \frac{\eta}{1 - \eta c_t} M'_t \right) \bU_t$. Thus the update rule here corresponds to the update rule for the matrix sensing case, and the rest of the proof follows from the proof of Theorem~\ref{thm:technical-main}.


\end{proof}

\section{Simulations}\label{sec:exp}

In this section, we present simulations to complement our theoretical results.
In the first experiment, we show that the generalization performance of gradient descent depends on the choice of initialization.
In particular, smaller initializations enjoy better generalization performance than larger initializations.
In the second experiment, we demonstrate that gradient descent can run for a large number of  iterations and the test error keeps decreasing, which suggests early stopping is not necessary. 
In the third experiment, we show that a natural projected gradient descent procedure works poorly compared to gradient descent on the factorized model, which suggests the power of using a factorized model.
In the last experiment, we report results for running stochastic gradient descent on the quadratic neural network setting, starting from a large initialization.

We generate the true matrix by sampling each entry of $\bUg$ independently from a standard Gaussian distribution and let $\bXg = \bUg \bUg^\top$. 
Each column of $\bUg$ is normalized to have unit norm, so that the spectral norm of $\bXg$ is close to one.
For every sensing matrix $A_i$, for $i = 1,\dots,m$, we sample the entries of $A_i$ independently from a standard Gaussian distribution.
Then we observe $b_i = \innerProduct{A_i}{\bXg}$.
When an algorithm returns a solution $\hat X$, we measure training error by:
\[ \sqrt{\frac{\sum_{i=1}^m (\innerProduct{A_i}{\hat X} - b_i)^2} {\sum_{i=1}^m b_i^2}.} \]
We measure test error by:
\[ \frac{\normFro{\hat X - \bXg}} {\normFro{\bXg}}. \]
For the same $\bXg$, we repeat the experiment three times, by resampling the set of sensing matrices $\{A_i\}_{i=1}^m$.
We report the mean and the error bar.

\paragraph{Choice of initialization.}
Let $U_0 = \alpha \Id$.
We use $m = 5dr$ samples, where rank $r = 5$.
We plot the training and test error for different values of $\alpha$.
Figure \ref{fig_init} shows that the gap between the training and test error
narrows down as $\alpha$ decreases.
We use step size $0.0025$ and run gradient descent for $10^4$ iterations.

\begin{figure}[!ht]
	\centering
	\includegraphics[width=0.5\textwidth]{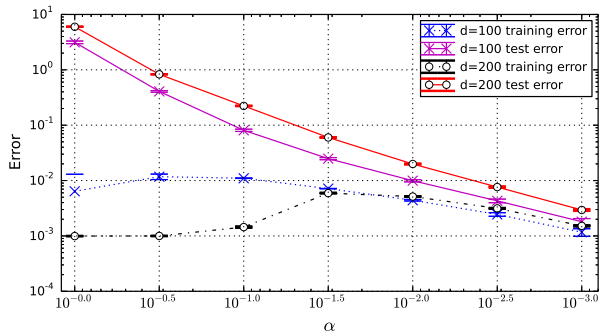}
	\caption{Generalization performance depends on the choice of initialization:
	the gap between training and test error decreases as $\alpha$ decreases.
	Here the number of samples is $5dr$, where rank $r = 5$. We initialize with
	$\alpha \Id$, and run $10^4$ iterations with step size $0.0025$.}
	\label{fig_init}
\end{figure}

In Figure \ref{fig_init_long}, we run for longer iterations to further compare the generalization performance of initialization $U_0 = \alpha\Id$ for $\alpha = 1.0, 10^{-3}$.
We report the mean values at each iteration over three runs.
When $\alpha = 1.0$, despite the training error decreases below $10^{-4}$,
the test error remains to be on the order of $10^{-1}$.

\begin{figure}[!ht]
	\centering
	\includegraphics[width=0.5\textwidth]{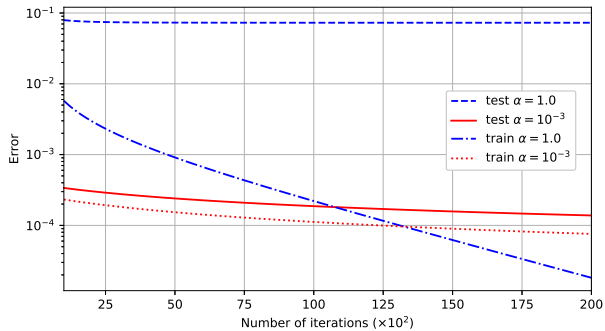}
	\caption{Further comparison between the generalization performance of large versus small initializations. We plot the data points from iteration 500 onwards to simplify the scale of the y-axis. The step size is $0.0025$.}
	\label{fig_init_long}
\end{figure}

\paragraph{Accuracy.}
In this experiment, we fix the initialization to be $U_0 = 0.01 \Id$,
and apply the same set of parameters as the first experiment.
We keep gradient descent running for $10^5$ iterations, to see if test error
keeps decreasing or diverges at some point.
Figure \ref{fig_accuracy} confirms that test error goes down gradually.

\begin{figure}[!ht]
	\centering
	\includegraphics[width=0.5\textwidth]{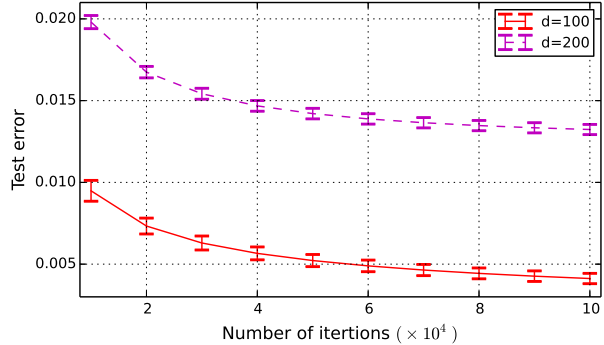}
	\caption{Test error keeps decreasing as the number of iterations goes to $10^5$. Here the number of samples is $m = 5dr$, where rank $r = 5$. Note that the initial test error is approximately $1$.}
	\label{fig_accuracy}
\end{figure}

\paragraph{Projected gradient descent.}
In this experiment, we consider the following natural projected gradient descent (PDG) procedure.
Let $f(X) = \frac 1 m \sum_{i=1}^m \left(\innerProduct{A_i}{X} - b_i \right)^2$.
At every iteration, we first take a gradient step over $f(X)$,
then project back to the PSD cone.
We consider the sample complexity of PGD by varying the number of sensing
matrices $m$ from $5d$ to $35d$.
Here the rank of $\bXg$ is $1$.
We found that the performance of PGD is much worse compared to gradient descent on the factorized model.
For both procedures, we use step size equal to $0.0025$.
We stop when the training error is less than $0.001$, or when the number of iterations reaches $10^4$.
Figure \ref{fig_pgd} shows that gradient descent on the factorized model
consistently recovers $\bXg$ accurately.
On the other hand, the performance of projected gradient descent gets even worse as $d$ increases from $100$ to $150$.

\begin{figure}[!ht]
	\centering
	\includegraphics[width=0.5\textwidth]{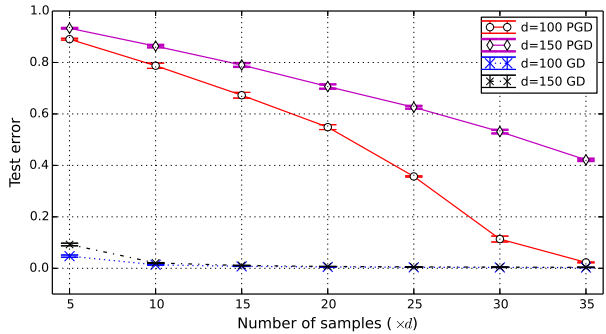}
	\caption{Projected gradient descent (PGD) requires more samples to recover $\bXg$ accurately,
	than gradient descent on the factorized model.
	Moreover, the performance of PGD gets worse as $d$ increases.}
	\label{fig_pgd}
\end{figure}

\paragraph{Stochastic gradient descent.}
In this experiment, we complement our theoretical results by running stochastic gradient descent from large initializations.
We generate $m = 5dr$ random samples and compute their true labels as the training dataset.
For stochastic gradient descent, at every iteration we pick a training data point uniformly at random from the training dataset.
We run one gradient descent step using the training data point.
We initialize with $U_0 = \Id$ and use step size $8 \times 10^{-5}$.
Figure \ref{fig_sgd} shows that despite the training error decreases to $10^{-7}$, the test error remains large.
We also report the results of running gradient descent on the same instance for comparison.
As we have already seen, Figure \ref{fig_gd_quadratic} shows that gradient descent also gets stuck at a point with large test error, despite the training error being less than $10^{-7}$.

\begin{figure}[!ht]
	\centering
	\begin{subfigure}{.49\textwidth}
  	\centering
	  \includegraphics[width=.9\linewidth]{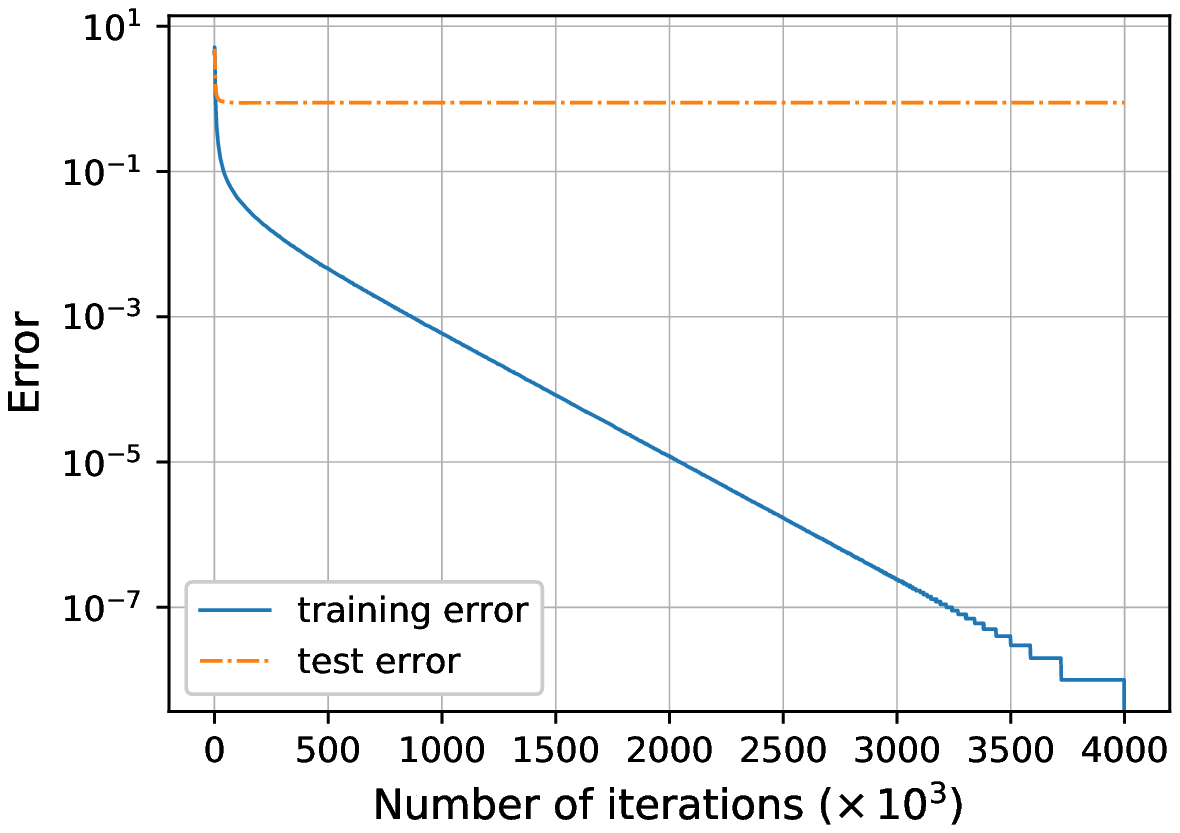}
		\caption{Stochastic gradient descent}
		\label{fig_sgd}
	\end{subfigure}
	\begin{subfigure}{.49\textwidth}
	  \centering
	  \includegraphics[width=.9\linewidth]{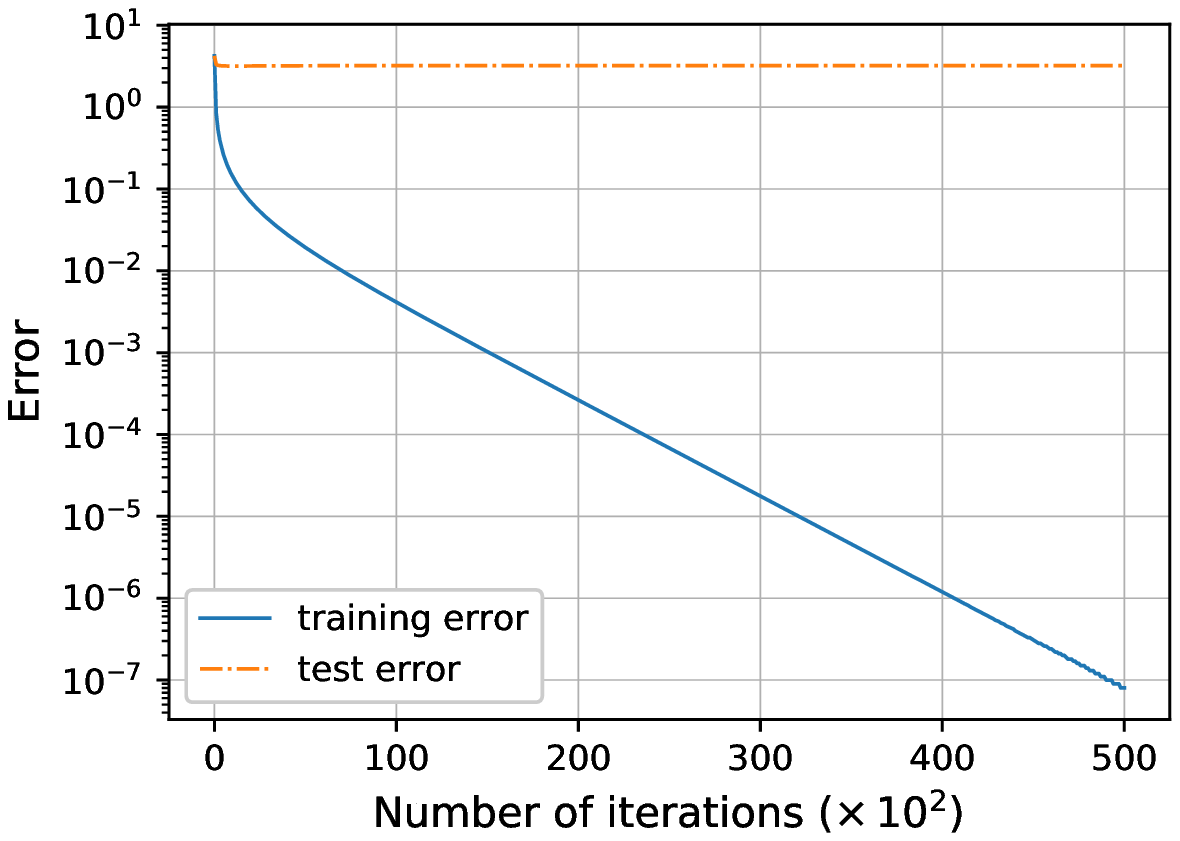}
	  \caption{Gradient descent}
	  \label{fig_gd_quadratic}
	\end{subfigure}
	\caption{Stochastic gradient descent, when initialized with the identity matrix, does not generalize to test data. Here $d = 100$ and $r = 5$.}
\end{figure}

\section{Conclusions}

The generalization performance of over-parameterized non-linear models, especially neural networks, has been a very intriguing research topic. This paper theoretically explains the regularization power of the optimization algorithms in learning matrix factorization models and one-hidden-layer neural nets with quadratic activations. In these cases, the gradient descent algorithm prioritizes to searching through the low complexity models. 

It's an very interesting open question to establish similar results for deeper neural networks with other activations (e.g., ReLU) and loss functions (e.g., logistic loss). We remark that likely such a result will require not only a better understanding of statistics, but also a deep grasp of the behavior of the optimization algorithms for non-linear models, which in turns is another fascinating open question.  

\subsection* {Acknowledgments: }

We thank Yuxin Chen, Yann Dauphin, Jason D. Lee, Nati Srebro, and Rachel A. Ward for insightful discussions at various stages of the work. 
Hongyang Zhang is supported by NSF grant 1447697.

\appendix
\section{Proof of Main Propositions}\label{sec:proofprop}

\subsection{Proof of Proposition~\ref{prop:fullrankerror}}
We start off with a straightforward triangle inequality for bounding $E_{t+1}$ given $E_t$.  
\begin{lem}
	By update rule (equation~\eqref{eqn:def-Ut}) and the definition of $E_t$ (equation~\eqref{eqn:def-ze}), we have that 
	\begin{align}
	E_{t+1} = (\Id-\Id_{S_{t+1}})(\Id - \eta M_t)E_t\mper\nonumber
	\end{align}
	It follows that 
	\begin{align}
	\Norm{E_{t+1}} \le \Norm{(\Id - \eta M_t)E_t} \le \Norm{E_t} + \eta \Norm{M_tE_t}\mper\label{eqn:33}
	\end{align}
\end{lem}
\noindent Therefore, next we will bound the norm of $M_tE_t$. The key idea is to use the restricted isometry property to control the effect of $M_t$ (see the meta claim equation~\eqref{eqn:meta} in Section~\ref{sec:prelim} for more intuitions.)
\begin{lem}\label{lem:norm_M_t}
	In the setting of Proposition~\ref{prop:fullrankerror}, we have that 
	\begin{align}
	\Norm{M_tE_t}\le \Norm{E_t} \left(\delta\normFro{Z_tZ_t^\top - \bXg} + 2\delta \normFro{Z_tE_t^\top} + \delta\norm{E_tE_t^\top}_{\star} + \norm{Z_t E_t^{\top}} + \Norm{\bXg(\Id - \Id_{S_t})}\right)\label{eqn:34}
	\end{align}
	As a direct consequence, using the assumption $\Norm{E_t}\le 1/d$,
	\begin{align}
	\Norm{E_{t + 1}} \leq \Norm{E_t}\left( 1 + O \left(\eta \delta \sqrt{r}  + \eta  \Norm{\bXg(\Id - \Id_{S_t})}  \right) \right)\label{eqn:31}
	\end{align}
\end{lem}

\begin{proof}
	We first note that $E_t = (\Id - \Id_{S_t})E_t$ by the update rule and definition~\ref{eqn:def-ze}. It follows that $M_tE_t = M_t (\Id - \Id_{S_t})E_t$. 
	Next, since $U_tU_t^\top - \bXg-E_tE_t^\top$ has rank at most $4r$, by Lemma~\ref{lem:property_1} we have that 
	\begin{align}
	& \Norm{\frac{1}{m}\sum_{i = 1}^m \inner{A_i, U_tU_t^\top - \bXg-E_tE_t^\top}A_i(\Id - \Id_{S_t})E_t - (U_tU_t^\top - \bXg-E_tE_t^\top)(\Id - \Id_{S_t})E_t} \nonumber \\ 
	\le &\, \delta \times \normFro{U_tU_t^\top - \bXg-E_tE_t^\top} \Norm{(\Id - \Id_{S_t})E_t} \label{place_1} \\
	=&\, \delta \times \left(\normFro{Z_tZ_t^\top - \bXg} + 2\norm{Z_tE_t^\top}_F \right)\Norm{E_t} \nonumber
	\end{align}
	Note that $(U_tU_t^\top - \bXg-E_tE_t^\top)(\Id - \Id_{S_t})E_t = Z_tE_t^\top E_t - \bXg(\Id - \Id_{S_t})E_t$. It follows that 
	\begin{align}
	& \Norm{\frac{1}{m}\sum_{i = 1}^m \inner{A_i, U_tU_t^\top - \bXg-E_tE_t^\top}A_i(\Id - \Id_{S_t})E_t} \nonumber \\
	&\le \Norm{\bXg(\Id - \Id_{S_t})E_t} + \Norm{Z_tE_t^\top E_t} + \delta\left(\normFro{Z_tZ_t^\top - \bXg} + 2\norm{Z_tE_t^\top}_F\right)\Norm{E_t} \nonumber \\
	& \le \Norm{E_t}\left(\delta\normFro{Z_tZ_t^\top - \bXg} + 2\delta \normFro{Z_tE_t^\top} + \norm{Z_tE_t^{\top}} + \Norm{\bXg(\Id - \Id_{S_t})}\right)\label{eqn:77}
	\end{align}
	By Lemma~\ref{lem:property_2}, we have that 
	\begin{align}
	\Norm{\frac{1}{m}\sum_{i = 1}^m \inner{A_i, E_tE_t^\top}A_i(\Id - \Id_{S_t})E_t - E_tE_t^\top (\Id - \Id_{S_t})E_t} \le \delta \norm{E_tE_t^\top}_{\star}\Norm{E_t}\label{eqn:6}
	\end{align}
	Combining equation~\eqref{eqn:6} and~\eqref{eqn:77} we complete the proof of equation~\eqref{eqn:34}. To prove equation~\eqref{eqn:31}, we will use equation~\eqref{eqn:33} and that $\normFro{Z_tZ_t^\top - \bXg}\lesssim \sqrt{r}$ (Corollary~\ref{cor:stat}),
	$\normFro{Z_t}^2 \le \sqrt{r} \times \normFro{Z_t Z_t^{\top}} \lesssim r$.
\end{proof}

\noindent When $\Norm{U_t U_t^{\top} - \bXg}$ is small, the growth of $\norm{E_t}$ becomes slower.
\begin{lem}\label{lem:norm_M_t_e}\label{lem_Mt}
	In the setting of Proposition~\ref{prop:fullrankerror}, we have that
	\begin{align*}
		\norm{M_t} \le \norm{U_tU_t^{\top} - \bXg} + \delta\normFro{U_tU_t^{\top} - \bXg} + 3\sqrt{r} \norm{E_t}.
	\end{align*}
	As a direct consequence, $\norm{M_t} \le O(1)$. And
	\begin{align}
		\Norm{M_tE_t} \lesssim \Norm{E_t} \left(\norm{U_tU_t^\top - \bXg}_F + \sqrt{r}\Norm{E_t}\right)\label{eqn:345124123}
	\end{align}
\end{lem}

\begin{proof}
	By the definition of $M_t$, we have that
	\begin{align}
		\norm{M_t} & = \bignorm{\frac 1 m \sum_{i=1}^m \innerProduct{A_i}{U_tU_t^{\top} - \bXg}A_i} \nonumber \\
		& = \bignorm{\frac 1 m \sum_{i=1}^m \innerProduct{A_i}{U_tU_t^{\top} - \bXg - E_tE_t^{\top}}A_i + \frac 1 m \sum_{i=1}^m \innerProduct{A_i}{E_tE_t^{\top}}A_i} \nonumber \\
		& \le \norm{U_tU_t^{\top} - \bXg} + \delta \times \left(\normFro{U_tU_t^{\top} - \bXg - E_tE_t^{\top}} + \normNuclear{E_t E_t^{\top}}\right) \label{eq:place_8} \\
		& \le (1 + \delta) \normFro{U_tU_t^{\top} - \bXg} + 3\sqrt r \norm{E_t} \nonumber
	\end{align}
	where the second to last line is because of Lemma \ref{lem:property_1} and Lemma \ref{lem:property_2}.
	For the last line, we use
	\[ \normFro{U_tU_t^{\top} - \bXg - E_tE_t^{\top}} = \normFro{Z_tZ_t^{\top} - \bXg} + 2 \normFro{Z_tE_t^{\top}} \lesssim \normFro{Z_tZ_t^{\top} - \bXg} + \sqrt r \norm{E_t}, \]
	because of Corollary \ref{cor:stat}.
	And $\Norm{Z_t Z_t^{\top}- \bXg}_F \leq \Norm{U_t U_t^{\top} - \bXg}_F  + O(\sqrt{r}\Norm{E_t})$, because
	\begin{align*}
		\Norm{U_t U_t^{\top} - \bXg}_F &= \Norm{(Z_t  +E_t) (Z_t + E_t)^{\top} - \bXg}_F \\
		&\geq \Norm{Z_t Z_t^{\top}- \bXg}_F - 2\Norm{Z_t}_F \Norm{E_t} - \Norm{E_t}\Norm{E_t}_F
	\end{align*}
\end{proof}

\noindent Finally we complete the proof of Proposition \ref{prop:fullrankerror}.
\begin{proof}[Proof of Proposition~\ref{prop:fullrankerror}]
	Using the fact that $\bXg$ has spectral norm less than 1, wee can bound the term $ \Norm{\bXg(\Id - \Id_{S_t})}$ in equation~\eqref{eqn:31} by
	\begin{align}
	\Norm{\bXg(\Id - \Id_{S_t})}  \le	\Norm{\bUg^\top (\Id - \Id_{S_t})} = \Norm{(\Id - \Id_{S_t}) \bUg }  = \sin(S_t, \bUg)\mper \nonumber
	\end{align}
	Since $S_t$ is the column span of $Z_t$, using the equation above and equation~\eqref{eqn:31} we conclude the proof. 
\end{proof}

\subsection{Proof of Proposition~\ref{prop:unknown}}

We first present a simpler helper lemma that will be used in the proof. 
One can view $X$ in the following lemma as a perturbation. The lemma bounds the effect of the perturbation to the left hand side of equation~\eqref{eqn:14}. 
\begin{lem}\label{lem:projection_change}
	Let $S \in \Real^{d \times r}$ be a column orthonormal matrix and
	$S^{\bot} = \Id - SS^{\top}$ be its orthogonal complement.
	Let $X \in \Real^{d\times d}$ be any matrix where $\norm{X} < \frac{1}{3}$.
	We have:
	\begin{align}
	\norm{\Id_{(\Id - X) S} (\Id - X) S^{\bot} + 2\Id_S X S^{\bot}}
	\lesssim \norm{X}^2. \label{eqn:14}
	\end{align} 	
\end{lem}

\begin{proof}
	By definition, we know that:
	\begin{align}
		\Id_{(\Id - X) S} &= (\Id - X) S \left( S^{\top}(\Id - X)^{\top} (\Id - X) S \right)^{-1}  S^{\top}(\Id - X)^{\top}	\nonumber \\
		&= (\Id - X) S \left(\Id - S^{\top} (X^{\top}  + X) S  + S^{\top} X^{\top} X S \right)^{-1} S^{\top}(\Id - X)^{\top} \label{eq_reduction}
	\end{align}
	Denote by $Y = S^{\top} (X^{\top} + X) S - S^{\top} X^{\top} X S$.
	We have that $\norm{Y}\le 2 \norm{X} + \norm{X}^2 < 7 \norm{X} / 3 $.
	By expanding $(\Id - Y)^{-1}$, we obtain:
	\begin{align*}
		\bignorm{(\Id - Y)^{-1} - (\Id + Y)}
		&\le \sum_{i=2}^{\infty} \norm{Y}^i \\
		&\le \frac {\norm{Y}^2} {1 - \norm{Y}} \le 25 \norm{X}^2.
	\end{align*}
	Hence we get:
	\[ \bignorm{(\Id - Y)^{-1} - (\Id + S^{\top} (X^{\top} + X) S)} \le 26 \norm{X}^2. \]
	Denote by
	\begin{align*}
		A &= (\Id - X) S \left(\Id + S^{\top} (X^{\top} + X) S\right) S^{\top} (\Id - X^{\top}) \\
			&= (\Id - X) \left(\Id_S + \Id_S (X^{\top} + X) \Id_S\right) (\Id - X^{\top}).
	\end{align*}
	We separate the terms in $A$ which has degree 1 in $X$:
	\begin{align*}
		A_1 &= \Id_S - X \Id_S - \Id_S X^{\top} + \Id_S (X^{\top} + X) \Id_S
	\end{align*}
	Consider the spectral norm of $A (\Id - X) S^{\bot}$.
	We have that $A_1 S^{\bot} = -\Id_S X S^{\bot}$.
	For $- A_1 X S^{\bot}$, the only term which has degree 1 in $X$ is $-\Id_S X S^{\bot}$.
	To summarize, we obtain by triangle inequality that:
	\begin{align*}
		\norm{\Id_{(\Id - X) S} (\Id - X) S^{\top} + 2\Id_S X S^{\bot}}
		&\le \norm{A + 2\Id_S X S^{\bot}} + 26 \norm{X}^2 \\
		&\lesssim \norm{A_1 + 2\Id_S X S^{\bot}} + \norm{X}^2 \lesssim \norm{X}^2
	\end{align*}
\end{proof}

\noindent Now we are ready present the proof of Proposition \ref{prop:unknown}.

\begin{proof}[Proof of Proposition~\ref{prop:unknown}]
	We first consider the term $\bM_t Z_t$:
	\begin{align}
		\bM_t \bZ_t &= \frac{1}{m}\sum_{i = 1}^m \langle \bA_i, \bU_t \bU_t^{\top} - \bXg \rangle \bA_i \bZ_t \nonumber
		\\
		& = \frac{1}{m}\sum_{i = 1}^m \left\langle \bA_i, Z_t Z_t^{\top} -\bXg\right\rangle   \bA_i \bZ_t 	\nonumber\\
		& + \frac{1}{m}  \sum_{i = 1}^m \left\langle \bA_i, \bE_t Z_t^{\top} + Z_t \bE_t^{\top}\right\rangle   \bA_iZ_t +  \frac{1}{m}\sum_{i = 1}^m \left\langle \bA_i,  \bE_t \bE_t^{\top}\right\rangle \bA_i \bZ_t  \label{eqn:15}
	\end{align}
	Note that $\nabla f(Z_t) =  \frac{1}{m}\sum_{i = 1}^m \left\langle \bA_i, Z_t Z_t^{\top} -\bXg\right\rangle   \bA_i \bZ_t$. 
	Using  Lemma~\ref{lem:property_1} and~\ref{lem:property_2} on the two terms in line~\eqref{eqn:15} with the fact that $Z_t^{\top} E_t = 0$, we have that
	\begin{align}
		&\left\|M_tZ_t - E_t Z_t^{\top} Z_t - \cG(Z_t)\right\|  \leq 
	2 \sqrt{r} \delta\| \bE_t\|\| Z_t\|^2 +  \delta d \| \bE_t \|^2 \|Z_t \|
		\label{eq_MtZt}
	\end{align}
We decompose $Z_{t+1}$ by: 
	\begin{align}
	Z_{t + 1} &= \Id_{S_{t + 1}} U_{t + 1} = U_{t + 1} - (\Id - \Id_{S_{t + 1}}) U_{t + 1} \nonumber
	\\
	& = (\Id - \eta M_{t }) Z_t + \left(\Id - \eta M_t \right) \bE_t  -  (\Id - \Id_{S_{t + 1}})  \left(\Id - \eta M_t \right) \bE_t  \nonumber
	\\
	&=  Z_t - \eta M_{t } Z_t  + \Id_{S_{t + 1}} ( \Id - \eta M_t ) E_t, \label{eq_Z_t+1}
	\end{align}
	where in the second equation we use the fact that
	$(\Id - \Id_{S_{t+1}}) (\Id - \eta M_t) Z_t = 0$,
	since the column subspace of $Z_t$ is $S_t$.
	Setting $S = \Id_{S_t}$ and $X = \eta M_t$ in Lemma \ref{lem:projection_change}, 
	we have that $\Id_{S_{t+1}} = \Id_{(\Id - X) S}$. Applying Lemma \ref{lem:projection_change}, we conclude:
	\begin{align*}
		\norm{\Id_{S_{t+1}} (\Id - \eta M_t) S^{\bot} + 2\Id_S (\Id - \eta M_t) S^{\bot}} \le \eta^2 \norm{M_t}^2,
	\end{align*}
	which implies that:
	\begin{align*}
		\norm{\Id_{S_{t+1}} (\Id - \eta M_t) E_t + 2\Id_S (\Id - \eta M_t) E_t} \le \eta^2 \norm{E_t} \norm{M_t}^2.
	\end{align*}
	Since $\norm{M_t} \le O(1)$ by Lemma \ref{lem_Mt},
	the conclusion follows by combining Equation \eqref{eq_MtZt} and \eqref{eq_Z_t+1} with the equation above.
\end{proof}

\subsection{Proof of Proposition~\ref{prop:low-rank-relative-error}}

Towards proving Proposition~\ref{prop:low-rank-relative-error}, we further decompose $Z_t$ into 
\begin{align}
Z_t = (\bUg + F_t)R_t \label{eqn:decompose_rankr}
\end{align}
where $R_t\in \R^{r\times d}, F_t\in \R^{d\times r}$ are defined as
\begin{align}
R_t = \bUg^\top Z_t, \textup{ and } F_t = (\Id- \Id_{\bUg})Z_t R_t^{+} \label{eqn:decompose_z}
\end{align}
Recall that $X^+$ denotes the pseudo-inverse of $X$. We first relate the the spectral norm of $F_t$ with our target in Proposition~\ref{prop:low-rank-relative-error}, the principal angle between $Z_t$ and $\bUg$.  Up to third order term, we show that $\|F_t\|$ is effective equal to the principle angle, and this is pretty much the motivation to decompose $Z_t$ in equation~\eqref{eqn:decompose_rankr}.

\begin{lem}\label{lem:angleF}
	Let $F_t$ be defined as in equation~\eqref{eqn:decompose_z}. Then, if $\Norm{F_t} < 1/3$, we have that 
	\begin{align}
	\Norm{F_t} - \Norm{F_t}^3 \le \sin(Z_t, \bUg) \le \Norm{F_t} \nonumber
	\end{align}
\end{lem}
\begin{proof}
	Let $S_t = (\bUg + F_t)(\Id + F_t^\top F_t)^{-1/2}$. By the fact that $\bUg^\top F_t = 0$, we have $S_t^\top S_t = \Id$ and $S_t$ has the same column span as $Z_t$. Therefore, the columns of $S_t$ form an orthonormal basis of $Z_t$, and we have that 
	\begin{align}
	\sin(Z_t, \bUg) = \Norm{(\Id - \Id_{\bUg})S_t} = \Norm{F_t(\Id + F_t^\top F_t)^{-1/2}}\nonumber
	\end{align}
	\sloppy Suppose $F_t$ has singular value $\sigma_j, j=1,\dots, r$, then it's straightforward to show that $F_t(\Id + F_t^\top F_t)^{-1/2}$ has singular values $\frac{\sigma_j}{\sqrt{1+\sigma_j^2}}, j=1,\dots, r$. The conclusion then follows basic calculus and the fact that $\max \sigma_j \le 1/3$. 
\end{proof}

Therefore, it suffices to bound the spectral norm of $F_t$.  However, the update rules of $F_t$ or $Z_t$ are difficult to reason about. Therefore, we introduce the following intermediate term $\tilde{Z}_t$ that bridges $Z_{t+1}, F_{t+1}$ with $Z_t$ and $F+t$. We define $\tilde{Z}_t$ as:
\begin{align}
\tilde{Z}_t = (\Id -  \eta E_t Z_t^{\top})Z_t ( \Id -  2 \eta Z_t^{+} \Id_{S_t} M_t E_t )
\end{align}

The motivation of defining such $Z_t$ is that $Z_{t+1}$ depends on $Z_t$ via relatively simple formula as the lemma below shows:
\begin{prop}\label{prop:update_z_t_1}
	In the setting of Proposition~\ref{prop:low-rank-relative-error},
\begin{align}\Norm{Z_{t + 1} - \left(\tilde{Z}_t - \eta \nabla f(Z_t) \right)}= O(\eta \tau_t)\nonumber\end{align}
\end{prop}

\noindent The proof of Proposition~\ref{prop:update_z_t_1} is deferred to the later part of this section. We also decompose $\tilde{Z}_t$ same as $Z_t$ to $\tilde{Z}_t = (\bUg + \tilde{F}_t) \tilde{R}_t$, where 
\begin{align}
\tilde{R}_t = \bUg^\top \tilde{Z}_t, \textup{ and } \tilde{F}_t = (\Id- \Id_{\bUg})\tilde{Z}_t \tilde{R}_t^{+} \label{eqn:decompose_z_tilde}
\end{align}

\noindent We will prove that $\tilde{R}_t$ is close to $R_t$ and $\tilde{F}_t$ is close to $F_t$ in the following sense:
\begin{prop}\label{prop:R_t_close}In the setting of Proposition~\ref{prop:low-rank-relative-error},
\begin{align}
\sigma_{\min}(\tilde{R}_t ) \geq \sigma_{\min}(R_t)\left( 1 - \frac{\eta}{100 \kappa} \right)\nonumber
\end{align}
\end{prop}
\begin{prop}\label{prop:F_t_close}In the setting of Proposition~\ref{prop:low-rank-relative-error},
\begin{align}
\norm{\tilde{F}_t - F_t} \lesssim \eta \Norm{E_t}  + \eta \rho\nonumber
\end{align}
\end{prop}

We will prove these propositions in the following sections. After that, we will focus on the update from $\tilde{Z}_t$ to $Z_{t + 1}$. In particular, we will bound $R_{t + 1} $ and $F_{t + 1}$ directly using $\tilde{R}_{t}$ and $\tilde{F}_t$.

\subsubsection{Proofs of Proposition~\ref{prop:update_z_t_1}, \ref{prop:R_t_close}, and \ref{prop:F_t_close}}

\begin{proof}[Proof of Proposition~\ref{prop:update_z_t_1}]
By definition of $\tilde{Z}_t$, 
\begin{align}
\tilde{Z}_t &=  (\Id -  \eta E_t Z_t^{\top})Z_t ( \Id -  2 \eta Z_t^{+} \Id_{S_t} M_t E_t )
\\
&= Z_t - \eta \left( E_t Z_t^{\top} Z_t + 2 \Id_{S_t} M_t E_t \right) + 2 \eta^2  E_t Z_t^{\top} \Id_{S_t} M_t E_t
\end{align}

Recall that by Proposition~\ref{prop:unknown}, the update rule of $Z_{t + 1}$ satisfies 
\begin{align}
\Norm{Z_{t + 1} - \left(Z_t  - \eta \nabla f(Z_t)- \eta E_t Z_t^{\top} Z_t - 2 \eta\Id_{S_t} M_t E_t  \right)} \leq \eta \tau_t\nonumber
\end{align}

Putting the above two formulas together (using the bound of $\Norm{M_t E_t}$ as in Lemma~\ref{lem:norm_M_t}) we have that
\begin{align}
\Norm{Z_{t + 1} - \left(\tilde{Z}_t - \eta \nabla f(Z_t) \right)} = O(\eta^2 \Norm{E_t}+  \eta \tau_t)
\end{align}

On the other hand, a direct calculation also shows that $\norm{\tilde{Z}_t - Z_t} = O(\eta \norm{E_t})$. Moreover, $\tilde{Z}_t$ is a rank $r$ matrix and $\Norm{Z_t}, \norm{\tilde{Z}_t} = O(1)$. Therefore, let us denote by $\tilde{Z}_t = Z_t + \Delta_t$, we have:
\begin{align*}
 \nabla f(\tilde{Z_t}) &= \frac{1}{m}  \sum_{i = 1}^m \langle A_i, \tilde{Z}_t \tilde{Z}_t^{\top}  - X^* \rangle A_i \tilde{Z}_t 
\\
& = \frac{1}{m}  \sum_{i = 1}^m \langle A_i, Z_t Z_t^{\top}  - X^* \rangle A_i \tilde{Z}_t  +  \frac{1}{m}  \sum_{i = 1}^m \langle A_i, Z_t \Delta_t^{\top} + \Delta_t Z_t^{\top} \rangle A_i \tilde{Z}_t  +  \frac{1}{m}  \sum_{i = 1}^m \langle A_i, \Delta_t \Delta_t^{\top} \rangle A_i \tilde{Z}_t 
\\
&= \frac{1}{m}  \sum_{i = 1}^m \langle A_i, Z_t Z_t^{\top}  - X^* \rangle A_i Z_t+  \frac{1}{m}  \sum_{i = 1}^m \langle A_i, Z_t Z_t^{\top}  - X^* \rangle A_i  \Delta_t
\\
&+  \frac{1}{m}  \sum_{i = 1}^m \langle A_i, Z_t \Delta_t^{\top} + \Delta_t Z_t^{\top} \rangle A_i \tilde{Z}_t  +  \frac{1}{m}  \sum_{i = 1}^m \langle A_i, \Delta_t \Delta_t^{\top} \rangle A_i \tilde{Z}_t 
\end{align*}

Note that $\Delta_t$ is at most rank $2r$, therefore, we can apply the RIP property of $\{ A_i \}_{i = 1}^m$ (Lemma~\ref{lem:property_1}) and $\norm{\tilde{Z_t}}, 
\norm{Z_t} = O(1)$ to conclude that 
\begin{align}
\frac{1}{m} \Norm{ \sum_{i = 1}^m \langle A_i, Z_t Z_t^{\top}  - X^* \rangle A_i  \Delta_t
+    \sum_{i = 1}^m \langle A_i, Z_t \Delta_t^{\top} + \Delta_t Z_t^{\top} \rangle A_i \tilde{Z}_t  +    \sum_{i = 1}^m \langle A_i, \Delta_t \Delta_t^{\top} \rangle A_i \tilde{Z}_t } = O(\Norm{\Delta_t})\nonumber
\end{align}

\noindent Recall that $\Norm{\Delta_t} = O(\eta \| E_t \|)$, the equation above and the formula for $\nabla f(\tilde{Z})$ implies that
\begin{align}
\Norm{Z_{t + 1} - \left(\tilde{Z}_t - \eta \nabla f(\tilde{Z_t}) \right)} = O(\eta^2 \Norm{E_t}+  \eta \tau_t) = O(\eta \tau_t)\nonumber
\end{align}

\end{proof}
\noindent
The above proposition implies that $Z_{t + 1}$ is very close to doing one step of gradient descent from $\tilde{Z}_t$. Thus, we will mainly focus on $\tilde{Z}_t$ in the later section. 

\begin{proof}[Proof of Proposition~\ref{prop:R_t_close}]
By definition of $\tilde{R}_t$, we know that 
\begin{align}
\tilde{R}_t &= \bUg^{\top} \tilde{Z}_t = \bUg^{\top}(\Id -  \eta E_t Z_t^{\top})Z_t ( \Id -  2 \eta Z_t^{+} \Id_{S_t} M_t E_t )\nonumber
\\
&= \bUg^{\top}Z_t ( \Id -  2 \eta Z_t^{+} \Id_{S_t} M_t E_t ) -   \eta \bUg^{\top} E_t Z_t^{\top} Z_t ( \Id -  2 \eta Z_t^{+} \Id_{S_t} M_t E_t ) \label{eqn:Rt}
\end{align}

By definition of $R_t$ and the assumption that $\| F_t \| \leq 1/3$, we have that $\sigma_{r}(Z_t) \geq \frac{1}{2} \sigma_{\min} (R_t)$, which implies that $\Norm{Z_t^+} \leq 2 \Norm{R_t^+}$. Thus,
using the bound of $\Norm{M_t E_t}$ as in Lemma~\ref{lem:norm_M_t},
\begin{align}
\Norm{Z_t^{+} \Id_{S_t} M_t E_t } \lesssim \frac{\Norm{ M_t E_t} }{\sigma_{\min}(R_t)} \lesssim \frac{\Norm{E_t} (\delta \sqrt{r} + \sin(Z_t, \bUg))}{\sigma_{\min}(R_t)}\nonumber
\end{align}

Similarly, we can get: 
\begin{align}
\norm{\bUg^{\top}E_t } = \Norm{\bUg^{\top}(\Id - \Id_{S_t})E_t }  \leq \Norm{E_t} \Norm{\bUg^{\top}(\Id - \Id_{S_t})} =  \Norm{E_t}  \sin(Z_t, \bUg)\label{eq:12345afjosji}
\end{align}
Therefore, bounding the terms in equation~\eqref{eqn:Rt} using the bounds above, we have:
\begin{align}
\sigma_{\min}(\tilde{R}_t ) &\geq \sigma_{\min}(\bUg^{\top}Z_t) (1 - \eta \Norm{ Z_t^{+} \Id_{S_t} M_t E_t}) - \eta \Norm{ \bUg^{\top} E_t Z_t^{\top} Z_t ( \Id -  2 \eta Z_t^{+} \Id_{S_t} M_t E_t )}\nonumber
\\
& \geq \sigma_{\min}(R_t)  (1 - \eta \Norm{ Z_t^{+} \Id_{S_t} M_t E_t}) - 2\eta \Norm{\bUg^{\top} E_t}\nonumber
\\
& \geq \sigma_{\min}(R_t)\left( 1 - \eta O\left(  \frac{\Norm{E_t} (\delta \sqrt{r} + \sin(Z_t, \bUg))}{\sigma_{\min}(R_t)} \right) \right) \label{eqn:201_rel}
\end{align}

Again, using $\sigma_{r}(Z_t) \geq \frac{1}{2} \sigma_{\min} (R_t)$ we know that 
\begin{align}
\frac{\Norm{E_t} (\delta \sqrt{r} + \sin(Z_t, \bUg))}{\sigma_{\min}(R_t)} &\leq 2\frac{\Norm{E_t} (\delta \sqrt{r} + \sin(Z_t, \bUg))}{\sigma_{\min}(Z_t)}\nonumber
\end{align}

By assumption of Proposition~\ref{prop:low-rank-relative-error} and that $\sigma_{\min}(Z_t) \geq \frac{1}{2} \Norm{E_t}$, we obtain:
\begin{align}
 2\frac{\Norm{E_t} (\delta \sqrt{r} + \sin(Z_t, \bUg))}{\sigma_{\min}(Z_t)}
 &\leq 2(\delta \sqrt{r} + \sin(Z_t, \bUg))\nonumber
 \\
 & \lesssim \sqrt{\rho} \tag{{by Assumption in proposition~\ref{prop:low-rank-relative-error}}}
\end{align}

Thus, by equation~\eqref{eqn:201_rel} above, and our choice of parameter such that $\sqrt{\rho} \lesssim 1/\kappa$, we conclude that
\begin{align}
\sigma_{\min}(\tilde{R}_t ) \geq \sigma_{\min}(R_t)\left( 1 - \frac{\eta}{100 \kappa} \right)\nonumber
\end{align}
\end{proof}

\noindent Next we prove prove Proposition~\ref{prop:F_t_close}, which focus on $\tilde{F}_t$:
\begin{proof}[Proof of Proposition~\ref{prop:F_t_close}]

We know that right multiply $Z_t 
$ by an invertible matrix does not change the column subspace as $Z_t$, thus won't change the definition of $F_t$, so we can just focus on $ (\Id -  \eta E_t Z_t^{\top})Z_t $. 
We know that 
\begin{align*}
\tilde{F}_t &= (\Id- \Id_{\bUg}) \left(\Id -  \eta E_t Z_t^{\top})Z_t (\bUg^{\top} (\Id -  \eta E_t Z_t^{\top})Z_t\right)^{+}
\\
 &= (\Id- \Id_{\bUg}) Z_t \left(\bUg^{\top} Z_t - \eta \bUg^{\top} E_t Z_t^{\top} Z_t \right)^{+}
 \\
 & - \eta  (\Id- \Id_{\bUg})E_t Z_t^{\top}\left(\bUg^{\top} Z_t - \eta \bUg^{\top} E_t Z_t^{\top} Z_t \right)^{+}
\end{align*}

Since $Z_t$ is rank $r$, we can do the SVD of $Z_t$ as $Z_t =  V \Sigma_Z W^{\top}$ for column orthonormal matrices $V, W \in \mathbb{R}^{d \times r}$ and diagonal matrix $\Sigma_Z \in \mathbb{R}^{r \times r}$. Then, we can write $\left(\bUg^{\top} Z_t - \eta \bUg^{\top} E_t Z_t^{\top} Z_t \right)^{+}$ as:
\begin{align*}
\left(\bUg^{\top} Z_t - \eta \bUg^{\top} E_t Z_t^{\top} Z_t \right)^{+} &= \left(\left(\bUg^{\top} V- \eta \bUg^{\top} E_t Z_t^{\top}V \right) \Sigma_Z W^{\top} \right)^{+}
\\
&=  W^{\top} \Sigma_Z^{-1} \left( \bUg^{\top}V - \eta \bUg^{\top}  \bE_t Z_t^{\top} V\right)^{-1}
\end{align*}

By our assumption that $\sin(Z_t, \bUg) \leq \frac{1}{3}$ we know that $\sigma_{\min}(\bUg^{\top} V) = \Omega(1)$. Thus, by $\Norm{Z_t} = O(1)$ and Woodbury matrix inversion formula we have:
\begin{align*}
\Norm{\left( \bUg^{\top}V - \eta \bUg^{\top}  \bE_t Z_t^{\top} V\right)^{-1} -  (\bUg^{\top}V )^{-1} } &\lesssim  \eta \Norm{\bUg^{\top}  \bE_t Z_t^{\top} V} \lesssim \eta \Norm{\bUg^{\top}  \bE_t}
\end{align*}

Recall that $F_t  = (\Id - \Id_{\bUg}) Z_t R_t^+ = (\Id - \Id_{\bUg}) Z_t W^{\top} \Sigma_Z^{-1} (\bUg^{\top} \bV)^{-1} $. Thus, using $\sigma_{r}(Z_t) \geq \frac{1}{2} \sigma_{\min} (R_t)$, an elementary calculation gives us:
\begin{align*}
\Norm{\tilde{F}_t - F_t} &\lesssim \eta \Norm{E_t} + \eta \frac{ \Norm{F_t} \Norm{\bUg^{\top} E_t } }{\sigma_{\min}(R_t)}
\end{align*}

By assumption $\Norm{E_t} \lesssim \sigma_{\min}(Z_t) $ in Proposition~\ref{prop:low-rank-relative-error}, together with $\sigma_{r}(Z_t) \geq \frac{1}{2} \sigma_{\min} (R_t)$ and inequality~\ref{eq:12345afjosji}, we have
\begin{align*}
\frac{ \Norm{F_t} \Norm{\bUg^{\top} E_t } }{\sigma_{\min}(R_t)} \lesssim \frac{  \Norm{F_t}\Norm{E_t}  \sin(Z_t, \bUg)}{\sigma_{\min}(Z_t)} \lesssim \Norm{F_t}^2
\end{align*}
By our choice of parameter, we know that $\Norm{F_t}^2 \leq \rho$, therefore, we proved that $\Norm{\tilde{F}_t - F_t}  \leq \eta( \| E_t\|_2 + \rho )$ as desired.
\end{proof}

Now, we can just focus on $\tilde{Z}_t$. One of the crucial fact about the gradient $\cG(\tilde{Z}_t)$ is that it can be decomposed into 
\begin{align*}
\cG(\tilde{Z}_t) = N_t \tilde{R}_t 
\end{align*}
where $N_t$ is a matrix defined as 
\begin{align}
N_t =  \frac{1}{m}\sum_{i = 1}^m\inner{\bA_i, \tilde{Z}_t\tilde{Z}_t^\top - \bXg}\bA_i(\bUg + \tilde{F}_t) \label{eqn:N_t}
\end{align}
Therefore, $\tilde{Z}_t$ and $\cG(\tilde{Z}_t)$ share the row space and we can factorize the difference between $\tilde{Z}_t$ and $\eta \cG_t(\tilde{Z}_t)$ as
\begin{align*}
\tilde{Z}_tN_t - \eta \cG(\tilde{Z}_t) = (\tilde{F}_t - \eta N_t)\tilde{R}_t
\end{align*}

Note that the definition of $N_t$ depends on the random matrices $A_1,\dots, A_t$. The following lemma show that for our purpose, we can essentially view $N_t$ as its population version --- the counterpart of $N_t$ when we have infinitely number of examples. The proof uses the RIP properties of the matrices $A_1,\dots, A_m$.

\begin{lem}\label{lem:M}
	In the setting of Proposition~\ref{prop:low-rank-relative-error}, let $N_t$ be defined as in equation~\eqref{eqn:N_t}. Then, 
	\begin{align*}
	\Norm{N_t - (\tilde{Z}_t\tilde{Z}_t^\top - X^\star)(\bUg + \tilde{F}_t)} \le 2\delta \Norm{\tilde{Z}_t\tilde{Z}_t^\top - X^\star}_F
	\end{align*}
\end{lem}
\begin{proof}
Recalling the definition of $N_t$, 
by Lemma~\ref{lem:property_1}, we have that 
\begin{align}\label{place_2}
\Norm{N_t - (\tilde{Z}_t\tilde{Z}_t^\top - X^\star)(\bUg + \tilde{F}_t)} & \le \delta \Norm{\tilde{Z}_t\tilde{Z}_t^\top - X^\star}_F \Norm{\bUg + \tilde{F}_t} \\
& \le 2\delta \Norm{\tilde{Z}_t\tilde{Z}_t^\top - X^\star}_F \tag{by the assumption $\Norm{\tilde{F}_t}< 1/3$}
\end{align}
\end{proof}

\begin{lem}\label{lem:Zt}
For any $t \geq 0$, suppose $\Norm{Z_{t + 1} - (\tilde{Z}_t - \eta \mathcal{G}(\tilde{Z}_t))} \leq \eta \tau$, we have
\begin{align*}
\| Z_{t + 1} \| \leq \| \tilde{Z}_t \| \left( 1 - \frac{1}{2}\eta \| \tilde{Z}_t\|^2 \right) + 2 \eta \| \tilde{Z}_t \| \| X^\star\| + \eta \tau\mper\nonumber
\end{align*}
\end{lem}
\begin{proof}
	By Lemma~\ref{lem:property_1}, we have: 
	\begin{align}\label{place_4}
	\Norm{\mathcal{G}(\tilde{Z}_t) - (\tilde{Z}_t \tilde{Z}_t^{\top} - X^\star )\tilde{Z}_t  } \leq \delta \Norm{\tilde{Z}_t \tilde{Z}_t^{\top} - X^\star}_{F} \Norm{\tilde{Z}_t }\mper 
	\end{align}
	Therefore, 
	\begin{align*}
	\Norm{Z_{t + 1}  }  &\leq \Norm{\tilde{Z}_t  - \eta \mathcal{G}(\tilde{Z}_t) } + \eta \tau 
	\\
	& \leq \Norm{\tilde{Z}_t   - \eta (\tilde{Z}_t \tilde{Z}_t^{\top} - X^\star )\tilde{Z}_t } + \eta \delta \Norm{\tilde{Z}_t \tilde{Z}_t^{\top} - X^\star}_{F} \Norm{\tilde{Z}_t } + \eta \tau
	\\
	& \leq \Norm{\left(\Id - \eta \tilde{Z}_t \tilde{Z}_t^{\top} + \eta X^\star \right)\tilde{Z}_t } + \eta \delta   \sqrt{r} \Norm{\tilde{Z}_t} \left(\Norm{\tilde{Z}_t}^2 + \Norm{X^\star} \right) + \eta \tau
	\\
	& \leq (1 - \eta \Norm{\tilde{Z}_t}) \Norm{\tilde{Z}_t } + \frac{1}{2}\eta\Norm{\tilde{Z}_t} \left( \Norm{\tilde{Z}_t}^2 + 4\Norm{X^\star}  \right)+ \eta \tau\tag{by $\delta \sqrt{r} \leq 1/2$}
	\end{align*}	
\end{proof}

As a direct corollary, we can inductive control the norm of $\tilde{Z}_t$. 
\begin{cor}\label{cor:stat}
	In the setting of Proposition~\ref{prop:low-rank-relative-error}, we have that 
	\begin{align*}
	\Norm{Z_{t+1}}\le 5, \Norm{\tilde{R}_t} \leq 6
	\end{align*}
	Moreover, 
	\begin{align*}
	\norm{Z_{t+1}Z_{t+1}^\top - X^\star}_F\lesssim \sqrt{r}, \textup{ and } \Norm{N_t}\lesssim \sqrt{r}
	\end{align*}
\end{cor}
\begin{proof}
	Using the assumption that $\Norm{\bXg}=1$ and the assumption that equation~\eqref{eqn:reduction} holds, then we have that $\Norm{\tilde{Z}_t} \leq \Norm{Z_t}(1 + O(\eta \Norm{E}_t)) \leq 5(1 + O(\eta \Norm{E}_t))$. Applying Lemma~\ref{lem:Zt} with $\tau = O(\tau_t)$, we have that $\Norm{Z_{t+1}}\le 5$. We also have that $\Norm{\tilde{R}_t}\le \Norm{\tilde{Z}_t}\le 6$. Moreover, we have $\norm{Z_{t+1}Z_{t+1}^\top - X^\star}_F \le \Norm{\tilde{Z}_t}_{\star} + \Norm{X^\star}_F \lesssim \sqrt{r}(\Norm{\tilde{Z}_t}+ \Norm{X^\star})\lesssim \sqrt{r}$. As a consequence, $\Norm{N_t}\leq (1 + \delta)\Norm{\tilde{Z}_t \tilde{Z}_t^{\top} - X^\star}_F\left( \Norm{U^\star} + \Norm{\tilde{F}_t} \right)\lesssim \sqrt{r}$. 
\end{proof}

We start off with a lemma that controls the changes of $\tilde{R}_t$ relatively to $R_{t+1}$.  
\begin{lem}\label{lem:RtRtplus1}
	In the setting of Proposition~\ref{prop:low-rank-relative-error}, then we have that $\tilde{R}_tR_{t+1}^{-1}$ can be written as
	\begin{align*}
		\tilde{R}_tR_{t+1}^{+} = \Id + \eta \tilde{R}_t\tilde{R}_t^\top - \eta \bSigmag + \xi_t^{(R)}
	\end{align*}
where $\norm{\xi_t^{(R)}} \lesssim \eta \delta\sqrt{r} + \eta \rho + \eta \norm{\tilde{F}_t} + \eta^2$. It follows that $\Norm{\tilde{R}_tR_{t+1}^+}\le 4/3$ and $\tau \leq 2 \rho \sigma_{\min}(R_{t+1})$. 
\end{lem}
\begin{proof}
	By the definition of $\tilde{R}_t$ and equation~\eqref{eqn:reduction}, we have that
	\begin{align}
	\eta \tau & \ge \|\bR_{t + 1} - \tilde{R}_t - \eta \bUg^{\top} \cG(\tilde{Z}_t) \| \nonumber\\
	& = \|\bR_{t + 1} - (\Id - \eta \bUg^\top N_t)\tilde{R}_t\| \label{eqn:1}
	\end{align}
Form this we can first obtain a very weak bound on $\sigma_{\min}(R_{t+1})$:
		\begin{align}
\sigma_{\min}(\bR_{t + 1}) & \ge  \sigma_{\min}((\Id - \eta \bUg^\top N_t)\tilde{R}_t) - \eta \tau  \nonumber\\
& \ge  \frac{3}{4}\cdot \sigma_{\min}(\tilde{R}_t) - O(\eta \rho \sigma_{\min}(\tilde{R}_t)) \tag{by $\tau \le O(\rho \sigma_{\min}(\tilde{R}_t))$ and $\norm{\eta \bUg^\top N_t}\le \Norm{\eta N_t}\lesssim \eta\sqrt{r}\le 1/4$}\\
&\ge \frac{1}{2}\cdot \sigma_{\min}(\tilde{R}_t)\label{eqn:weakbound}
	\end{align}
	Re-using equation~\eqref{eqn:1}, we have 
		\begin{align}
\|\Id - (\Id - \eta \bUg^\top N_t)\tilde{R}_tR_{t+1}^+\|& \le \eta \tau\cdot \Norm{R_{t+1}^+} = \eta \tau/\sigma_{\min}(R_{t+1})\nonumber\\
 &\le 2\eta \tau/\sigma_{\min}(\tilde{R}_t) \le O(\eta \rho)\label{eqn:3}
\end{align}
where we used equation~\eqref{eqn:weakbound} and $\tau \le O(\rho \sigma_{\min}(\tilde{R}_t))$. This also implies a weak bound for $\tilde{R}_tR_{t+1}^+$ that $\norm{\tilde{R}_tR_{t+1}^+}\le 2$. 
By Lemma~\ref{lem:M}, we have that 
\begin{align*}
\norm{\eta\bUg^\top N_t -\eta \bUg^\top (\tilde{Z}_t\tilde{Z}_t^\top - X^\star)(\bUg + \tilde{F}_t)} \le 2\delta \eta \norm{\tilde{Z}_t\tilde{Z}_t^\top - X^\star}_F\lesssim \delta \eta \sqrt{r}
\end{align*}
where we use the fact that $\norm{\tilde{Z}_t\tilde{Z}_t^\top - X^\star}_F\lesssim \sqrt{r}$.  Note that $ \bXg = \bUg \bSigmag\bUg^\top $ and $\tilde{Z}_t = (\bUg + \tilde{F}_t)\tilde{R}_t$,  we have $\norm{\eta\bUg^\top N_t -\eta (\tilde{R}_t\tilde{R}_t^\top-\bSigmag) (\bUg+\tilde{F}_t)^\top(\bUg +\tilde{F}_t)} \lesssim \delta \eta \sqrt{r}$. 

Bounding the higher-order term, we have that 
\begin{align*}
\Norm{\eta (\tilde{R}_t\tilde{R}_t^\top-\bSigmag)(\bUg+\tilde{F}_t)^\top (\bUg + \tilde{F}_t)- \eta (\tilde{R}_t\tilde{R}_t^\top-\bSigmag)}\lesssim \eta \Norm{\tilde{F}_t} \tag{by $\Norm{\tilde{R}_t}\le 6$}
\end{align*}
which implies that 
\begin{align}
 \norm{\eta\bUg^\top N_t -\eta (\tilde{R}_t\tilde{R}_t^\top-\bSigmag)} \lesssim \delta \eta \sqrt{r} + \eta \norm{\tilde{F}_t}\label{eqn:2}
\end{align}
Combining the equation above with equation~\eqref{eqn:3} and $\norm{\tilde{R}_tR_{t+1}^+}\le 2$, we have that 
	\begin{align}
\|\Id - (\Id - \eta \tilde{R}_t\tilde{R}_t^\top + \eta \bSigmag) \tilde{R}_tR_{t+1}^+\|& \lesssim \eta\delta\sqrt{r}  + \eta \rho + \eta\Norm{\tilde{F}_t} \label{eqn:21}
\end{align}
For $\eta \lesssim 1$, we know that 
\begin{align*}
\left\|(\Id - \eta \tilde{R}_t\tilde{R}_t^\top + \eta \bSigmag)  (\Id + \eta \tilde{R}_t\tilde{R}_t^\top - \eta \bSigmag) - \Id  \right\| \lesssim \eta^2 
\end{align*}
This implies that 
\begin{align*}
\|(\Id + \eta \tilde{R}_t\tilde{R}_t^\top - \eta \bSigmag) - \tilde{R}_tR_{t+1}^+\|& \lesssim \eta\delta\sqrt{r} + \eta \rho + \eta\Norm{\tilde{F}_t} + \eta^2  
\end{align*}
which completes the proof. 
\end{proof}

We express $F_{t+1}$ as a function of $\tilde{F}_t$ and other variables. 
\begin{lem}\label{lem:FtFtplus1}
In the setting of Proposition~\ref{prop:low-rank-relative-error}, let $N_t$ be defined as in equation~\eqref{eqn:N_t}. Then, 
\begin{align*}
F_{t+1} = \tilde{F}_t(\Id -\eta \tilde{R}_t\tilde{R}_t^\top )\tilde{R}_tR_{t+1}^{+} + \xi_t^{(F)}
\end{align*}	
where 
$
\norm{\xi_t^{(F)}}\lesssim \delta\eta\sqrt{r} + \eta  \rho+ \eta \Norm{\tilde{F}_t}^2
$.
\end{lem}

\begin{proof}
By equation~\eqref{eqn:reduction}, we have that 
\begin{align*}
\Norm{(\Id - \Id_{\bUg})\cdot \left(Z_{t+1} - (\tilde{Z}_t - \eta \cG(\tilde{Z}_t))\right)} 	\le \eta \tau_t
\end{align*}
which, together with the decomposition~\eqref{eqn:decompose_z}, implies  
\begin{align*}
\eta \tau &\ge 	\left\|\bF_{t + 1}\bR_{t + 1} -  \tilde{F}_t \tilde{R}_t + \eta (\Id - \Id_{\bUg}) \cG(\tilde{Z}_t) \right\| \\
&= 	\left\|\bF_{t + 1}\bR_{t + 1} -  (\tilde{F}_t -\eta (\Id - \Id_{\bUg}) N_t)\tilde{R}_t \right\| 
\end{align*}
Recall that $\tau_t \le 2\rho\sigma_{\min}(R_{t+1})$ (by Lemma~\ref{lem:RtRtplus1}), we conclude
\begin{align}
\left\|\bF_{t + 1}-  (\tilde{F}_t -\eta(\Id - \Id_{\bUg})  N_t)\tilde{R}_t R_{t+1}^{+}\right\|  \le 2\eta\rho  \label{eqn:5}
\end{align}
Note that $(\Id -\Id_{\bUg})X^\star = 0$ and that $(\Id -\Id_{\bUg})\tilde{Z}_t\tilde{Z}_t^\top = \tilde{F}_t\tilde{R}_t\tilde{R}_t^\top(\bUg+\tilde{F}_t)^\top$. We obtain that 

\begin{align}
\Norm{(\Id -\Id_{\bUg})N_t-  \tilde{F}_t\tilde{R}_t\tilde{R}_t^\top(\bUg+\tilde{F}_t)^\top(\bUg + \tilde{F}_t)} & \le 2\delta \Norm{\tilde{Z}_t\tilde{Z}_t^\top - X^\star}_F \lesssim \delta \sqrt{r} \label{eqn:4}
\end{align}
where we used the fact that $\Norm{\tilde{Z}_t\tilde{Z}_t^\top - X^\star}_F  \lesssim \sqrt{r}$ (by Lemma~\ref{cor:stat}). 

Bounding the higher-order term, we have that
\begin{align*}
\Norm{\tilde{F}_t\tilde{R}_t\tilde{R}_t^\top(\bUg+\tilde{F}_t)^\top (\bUg + \tilde{F}_t)- \tilde{F}_t\tilde{R}_t\tilde{R}_t^\top}\lesssim \Norm{\tilde{F}_t}^2 \tag{by $\Norm{\tilde{R}_t}\le 6$ from Corollary~\ref{cor:stat}}
\end{align*}
Combining equation~\eqref{eqn:5},~\eqref{eqn:4} and the equation above, and using the fact that $\Norm{\tilde{R}_tR_{t+1}^+}\le 2$, we complete the proof.\end{proof}

Combining Lemma~\ref{lem:FtFtplus1} and Lemma~\ref{lem:RtRtplus1}, we can relate the $\Norm{F_{t+1}}$ with $\Norm{\tilde{F}_t}$:
\begin{lem}\label{lem:FtFtplus1strong}
	In the setting of Proposition~\ref{prop:low-rank-relative-error}, we have that $\tilde{F}_t$ can be written as 
	\begin{align*}
	F_{t+1} = \tilde{F}_t(\Id - \eta \bSigmag) + \xi_t
	\end{align*}
	where $\Norm{\xi_t} \le O(\eta\delta\sqrt{r}  + \eta \Norm{\tilde{F}_t}^2 +  \eta \rho + \eta^2)$. As a consequence, 
	\begin{align*}
	\Norm{F_{t+1}} \le \Norm{\tilde{F}_t} + O(\eta \Norm{\tilde{F}_t}^2 + \eta \rho)
	\end{align*}
\end{lem}
\begin{proof}
	Combine Lemma \ref{lem:FtFtplus1} and Lemma \ref{lem:RtRtplus1}, we have that 
	\begin{align*}
	F_{t + 1} &= \tilde{F}_t(\Id -\eta \tilde{R}_t\tilde{R}_t^\top )\tilde{R}_tR_{t+1}^{+} + \xi_t^{(F)}
	\\
	&= \tilde{F}_t(\Id -\eta \tilde{R}_t\tilde{R}_t^\top)\left(\Id + \eta \tilde{R}_t\tilde{R}_t^\top - \eta \bSigmag + \xi_t^{(R)} \right) + \xi_t^{(F)}
	\end{align*}
	
	Thus, with the bound on $\norm{\xi_t^{(R)} }$ and $\norm{\xi_t^{(F)} }$ from Lemma~\ref{lem:FtFtplus1} and Lemma~\ref{lem:RtRtplus1}, we know that 
	\begin{align}
	\Norm{F_{t + 1} - (\Id - \eta \bSigmag) \tilde{F}_t} \lesssim  \eta \norm{\tilde{F}_t} +  \norm{\tilde{F}_t} \norm{\xi_t^{(R)}}  + \norm{\xi_t^{(F)}} &\lesssim \eta \norm{\tilde{F}_t}^2 + \eta \delta \sqrt{r} + \eta  \rho + \eta^2\nonumber\\
	& \lesssim \eta \Norm{\tilde{F}_t}^2 + \eta \rho\mper\nonumber
	\end{align}
\end{proof}

The proof of Proposition~\ref{prop:low-rank-relative-error} follow straightforwardly from Lemma~\ref{lem:angleF} and Lemma~\ref{lem:FtFtplus1strong}. 

\begin{proof}[Proof of Proposition~\ref{prop:low-rank-relative-error}]
	Using the assumption that $\Norm{F_t} \lesssim \sqrt{\rho}$ (Thus $\Norm{F_t}^2 \lesssim \rho$). Since we have showed that $\Norm{\tilde{F_t} - F_t} \lesssim \eta (\rho + \Norm{E_t}) $, the proof of this proposition followings immediately from Lemma \ref{lem:FtFtplus1strong}. 
\end{proof}

\subsection{Proof of Proposition~\ref{prop:convergence-and-eigen-grow}}

We first prove the following technical lemma that characterizes how much the least singular value of a matrix changes when it got multiplied by matrices that are close to identity. 
\begin{lem}\label{lem:eigen_upper}
Suppose $Y_1\in \mathbb{R}^{d\times r}$ and $\Sigma$ is a PSD matrix in $\mathbb{R}^{r \times r}$. For some $\eta > 0$, let 
$$Y_2 =  (\Id + \eta \Sigma)Y_1 (\Id - \eta Y_1^{\top}  Y_1)$$
Then, we have:  
\begin{align*}
\sigma_{\min} (Y_2) \geq \left(1 + \eta \sigma_{\min}(\Sigma) \right)\left(1 - \eta \sigma_{\min}(Y_1)^2\right) \sigma_{\min}(Y_1)
\end{align*}
\end{lem}

\begin{proof}
First let's consider the matrix  $Y := Y_1 (\Id - \eta Y_1^{\top}  Y_1) $. We have that that $\sigma_{\min} \left(  Y \right) = \left(1 - \eta \sigma_{\min}(Y_1)^2\right) \sigma_{\min}(Y_1)$. Next, we bound the least singular value of $(\Id + \eta \Sigma)Y$: 
\begin{align*}
\sigma_{\min}((\Id + \eta \Sigma)Y)&\ge \sigma_{\min}(\Id + \eta\Sigma)\sigma_{\min}(Y)= (1+\sigma_{\min}(\Sigma))\sigma_{\min}(Y)
\end{align*}
where we used the facts that $\sigma_{\min}(AB)\ge \sigma_{\min}(A)\sigma_{\min}(B)$, and that for any symmetric PSD matrix $B$, $\sigma_{\min}(\Id + B) = 1 + \sigma_{\min}(B)$. 
\end{proof}
Now we are ready to prove Proposition~\ref{prop:convergence-and-eigen-grow}. Note that the least singular value of $Z_t$ is closely related to the least singular value of $\tilde{R}_t$ because $F_t$ is close to 0. Using the machinery in the proof of Lemma~\ref{lem:RtRtplus1}, we can write $R_{t+1}$ as some transformation of $\tilde{R}_t$, and then use the lemma above to bound the least singular value of $R_{t+1}$ from below. 
\begin{proof}[Proof of Proposition~\ref{prop:convergence-and-eigen-grow}]
Recall that in equation~\eqref{eqn:21} in the the proof of Lemma~\ref{lem:RtRtplus1}, we showed  that 
\begin{align*}
\|\Id - (\Id - \eta \tilde{R}_t\tilde{R}_t^\top + \eta \bSigmag) \tilde{R}_tR_{t+1}^+\|& \lesssim \eta\delta\sqrt{r}  + \eta \rho + \eta\Norm{F_t}\mper\nonumber
\end{align*}
Together with $\Norm{\tilde{R}_t} \leq 6$ and $\Norm{\tilde{R}_t R_{t + 1}^+} \leq 2$ by Corollary~\ref{cor:stat} and Lemma~\ref{lem:RtRtplus1} respectively, we have:
\begin{align*}
\|\Id- (\Id + \eta \bSigmag) \tilde{R}_t (\Id - \eta \tilde{R}_t^{\top}\tilde{R}_t )  R_{t + 1}^+\|& \lesssim \eta \delta\sqrt{r}  + \eta \rho + \eta\Norm{F_t} + \eta^2\mper
\end{align*}
Denote by  $\xi = \Id- (\Id + \eta \bSigmag) \tilde{R}_t (\Id - \eta \tilde{R}_t^{\top}\tilde{R}_t )  R_{t + 1}^+$, we can rewrite
\begin{align*}
R_{t + 1} - (\Id + \eta \bSigmag) \tilde{R}_t (\Id - \eta \tilde{R}_t^{\top}\tilde{R}_t )   = \xi R_{t + 1} \mcom
\end{align*}
Without loss of generality, let us assume that $\sigma_{\min}(\tilde{R}_t) \leq \frac{1}{1.9 \sqrt{\kappa}}$.  By Lemma~\ref{lem:eigen_upper} that
\begin{align*}
\sigma_{\min}(R_{t + 1}) \geq \frac{(1 + \eta \sigma_{\min} (\bSigmag)) \left(1 - \eta \sigma_{\min}(\tilde{R}_t)^2\right) \sigma_{\min}(\tilde{R}_t) }{1  +O\left( \eta \delta\sqrt{r}  +\eta \rho +\eta\Norm{F_t} + \eta^2 \right) }\mper
\end{align*}
Since $\Norm{F_t} \lesssim \eta \rho t$ by Proposition~\ref{prop:low-rank-relative-error}, we have:
\begin{align*}
\sigma_{\min}(R_{t + 1}) & \geq (1 + \eta / \kappa) \left(1 - \eta \sigma_{\min}(\tilde{R}_t)^2\right) \sigma_{\min}(\tilde{R}_t) \left(1 - O \left( \eta \delta\sqrt{r}  + \eta \rho + \eta^2\rho t \right) \right)\\
& \ge \left(1 + \eta  \left(\frac{1} {3\kappa} -  O \left( \delta\sqrt{r}  +  \rho + \eta(\delta \rho) t \right) \right)  \right)\sigma_{\min}(\tilde{R}_t)  \tag{by $ \sigma_{\min}(\tilde{R}_t) \leq \frac{1}{1.9\sqrt{\kappa}}$}\\
& \ge \left(1 +   \frac{\eta} {4\kappa}   \right)\sigma_{\min}(\tilde{R}_t)  \tag{ by $t \le \frac{c}{\eta \kappa  \rho}$}
\end{align*}
Since  $\sigma_{\min}(\tilde{R}_t ) \geq \sigma_{\min}(R_t)\left( 1 - \frac{\eta}{100 \kappa} \right)$ by Proposition \ref{prop:R_t_close},
the conclusion follows.
\end{proof}

\subsection{Proof of Proposition ~\ref{prop:convergence}}\label{sec:prop:convergence}
	Since $U_{t+1} = U_t - \eta \cG(U_t)$, we first write down the error for step $t+1$:
	\begin{align}
		& \bignormFro{U_{t+1}U_{t+1}^{\top} - \bXg}^2
		= \bignormFro{(U_t - \eta \cG(U_t)) (U_t^{\top} - \eta \cG(U_t)^{\top}) - \bXg}^2 \nonumber \\
		& = \normFro{U_t U_t^{\top} - \bXg}^2 - 2\eta \innerProduct{\cG(U_t) U_t^{\top}}{U_t U_t^{\top} - \bXg} \label{eq_conv_deg1} \\
		& + \innerProduct{2\eta^2 (U_tU_t^{\top} - \bXg)} {\cG(U_t)\cG(U_t)^{\top}} + 4\eta^2 \normFro{\cG(U_t)U_t^{\top}}^2 \label{eq_conv_deg2} \\
		& - \innerProduct{4\eta^3 \cG(U_t) U_t^{\top}}{\cG(U_t)\cG(U_t)^{\top}}
		+ \eta^4 \normFro{\cG(U_t)\cG(U_t)^{\top}}^2 \label{eq_conv_deg3}
	\end{align}
	Denote by
	\[ \Delta = U_tU_t^{\top} - E_tE_t^{\top} = E_t Z_t^{\top} + Z_t E_t^{\top} + E_tE_t^{\top}. \]
	First we have that
	\[ \Norm{\Delta} \leq 2 \| E_t \| \| Z_t \| + \| E_t \|_2^2  \le O(\Norm{E_t}). \]
	We first consider the degree one term of $\eta$ in equation \eqref{eq_conv_deg1}.
	\begin{claim}\label{claim_conv_deg1}
		In the setting of this subsection, we have that:
		\[ \innerProduct{\cG(U_t)U_t^{\top}}{U_tU_t^{\top} - \bXg}
		\ge (1 - O(\eta)) \normFro{(Z_tZ_t^{\top} - \bXg)Z_t}^2
		- O(\delta) \normFro{Z_tZ_t^{\top} - \bXg}^2  - O(\sqrt{dr} \norm{E_t}). \]
	\end{claim}
	\begin{proof}
		First, we have that
		\begin{align}
			\innerProduct{\cG(U_t) U_t^{\top}}{U_tU_t^{\top} - \bXg}
			\ge \innerProduct{\cG(U_t) U_t^{\top}}{Z_tZ_t^{\top} - \bXg} - \normFro{\cG(U_t) U_t^{\top}} \normFro{\Delta} \label{eq_error1}
		\end{align}
		Since $\norm{U_t} \le O(1)$, we focus on the gradient of $U_t$.
		We divide $\cG(U_t)$ into the sum of three parts:
		\begin{align*}
			& Y_1 = \frac 1 m \sum_{i=1}^m \innerProduct{A_i}{Z_tZ_t^{\top} - \bXg} A_i Z_t, \\
			& Y_2 = \frac 1 m \sum_{i=1}^m \innerProduct{A_i}{Z_tZ_t^{\top} - \bXg} A_i E_t, \\
			& Y_3 = \frac 1 m \sum_{i=1}^m \innerProduct{A_i}{\Delta} A_i U_t
		\end{align*}
		By Lemma \ref{lem:property_1}, we have that:
		\begin{align}
			\norm{Y_1} \le \bignorm{(Z_tZ_t^{\top} - \bXg) Z_t} + \delta \times \normFro{Z_tZ_t^{\top} - \bXg} \norm{Z_t} \le O(1),\label{eq:y_1}
		\end{align}
		where $\normFro{Z_tZ_t^{\top} - \bXg} \le O(\sqrt r)$ by Corollary \ref{cor:stat}. Since $Y_1$ has rank at most $r$, we get that $\normFro{Y_1} \le O(\sqrt r)$. For $Y_2$, we apply Lemma \ref{lem:property_1} again:
		\begin{align*}
			\norm{Y_2} \le \norm{Z_tZ_t^{\top} - \bXg} \norm{E_t} + \delta \times \normFro{Z_tZ_t^{\top} - \bXg} \norm{E_t} \le O(\norm{E_t}).
		\end{align*}
		By assumption $\norm{E_t} \le 1 / d$, hence $\normFro{Y_2} \le O(1)$.
		Finally we apply Lemma \ref{lem:property_2} for $Y_3$:
		\begin{align*}
			\norm{Y_3} \le \norm{\Delta} \norm{U_t} + \delta \times (2\normFro{E_tZ_t^{\top}} + \normNuclear{E_tE_t^{\top}}) \norm{U_t} \le O(\norm{E_t})
		\end{align*}
		To summarize, we have shown that $\normFro{\cG(U_t)} \le O(\sqrt r)$.
		Back to equation \eqref{eq_error1}, we have shown that:
		\begin{align*}
			\normFro{\cG(U_t) U_t^{\top}} \normFro{\Delta} \le O(\sqrt{dr}) \norm{E_t}.
		\end{align*}
						For the other part in equation \eqref{eq_error1},
		\begin{align*}
			\innerProduct{\cG(U_t)U_t^{\top}}{Z_tZ_t^{\top} - \bXg}
			&\ge \innerProduct{M_t Z_tZ_t^{\top}} {Z_tZ_t^{\top} - \bXg}
			- \normFro{M_t \Delta} {\normFro{Z_tZ_t^{\top} - \bXg}} \\
		&\ge \innerProduct{M_t Z_tZ_t^{\top}} {Z_tZ_t^{\top} - \bXg} - O(\sqrt{dr}) \norm{E_t},
		\end{align*}
		because $\norm{M_t} \le O(1)$ from Lemma~\ref{lem_Mt}. Lastly, we separate out the $\Delta$ term in $M_t$ as follows:
		\begin{align}
			\innerProduct{M_t Z_tZ_t^{\top}}{Z_tZ_t^{\top} - \bXg}
			= & \innerProduct{\cG(Z_t)Z_t}{Z_tZ_t^{\top} - \bXg} + \label{eq_error2_p1}\\
			& \frac 1 m \sum_{i=1}^m \innerProduct{A_i}{E_tZ_t^{\top} + Z_tE_t^{\top}} \innerProduct{A_i}{Z_tZ_t^{\top} (Z_tZ_t^{\top} - \bXg)} + \label{eq_error2_p2} \\
			& \frac 1 m \sum_{i=1}^m \innerProduct{A_i}{E_tE_t^{\top}} \innerProduct{A_i}{Z_tZ_t^{\top} (Z_tZ_t^{\top} - \bXg)} \label{eq_error2_p3}
		\end{align}
		Since $Z_t$ is a rank-$r$ matrix, by Lemma \ref{lem:RIP3}, Equation \eqref{eq_error2_p1} is at least:
		\begin{align}
			\left\langle Z_t Z_t^{\top} - \bXg , Z_t Z_t^{\top} \left(Z_t Z_t^{\top} - \bXg \right)\right\rangle - \delta \norm{Z_t}^2 \normFro{Z_t Z_t^{\top }- \bXg}^2. \label{eq_zt_signal}
		\end{align}
		Equation \eqref{eq_error2_p2} is similarly bounded by $O(\sqrt{dr} \norm{E_t})$ using Lemma \ref{lem:RIP3}, since $E_tZ_t^{\top}$ also has rank at most $r$.
		Finally for equation \eqref{eq_error2_p3}, while $E_tE_t^{\top}$ may have rank $d$, we can still apply Lemma \ref{lem:RIP4} and afford to lose a factor of $d$, because $\norm{E_tE_t^{\top}} \le O(1 / d^2)$.
		To summarize, we have shown that:
		\[ \innerProduct{\cG(U_t) U_t^{\top}}{U_tU_t^{\top} - \bXg}
	\ge \innerProduct{\cG(Z_t) Z_t^{\top}}{Z_tZ_t^{\top} - \bXg} - O(\sqrt{dr}) \norm{E_t}. \]
	The conclusion follows by combining the above equation with \eqref{eq_zt_signal}.
	\end{proof}
	Next we work out the degree two term in equation \eqref{eq_conv_deg3}.	
	\begin{claim}\label{claim_conv_deg2}
		In the setting of this subsection, we have that both
		\begin{align*}
		 & |\innerProduct{(U_tU_t^{\top} - \bXg)}{\cG(U_t)\cG(U_t)^{\top}}|
		 \mbox{ and }
		 \normFro{\cG(U_t) U_t^{\top}}^2 \\
		 & \lesssim \normFro{(Z_tZ_t^{\top} - \bXg) Z_t}^2 + \delta \normFro{Z_tZ_t^{\top} - \bXg}^2 + r \norm{E_t}
		\end{align*}
	\end{claim}
	\begin{proof}
		The idea of the proof is similar to that for equation \eqref{eq_conv_deg1} and \eqref{eq_conv_deg2}. First of all, we have
		\begin{align*}
			\innerProduct{U_tU_t^{\top} - \bXg}{\cG(U_t) \cG(U_t)^{\top}}
			\ge \innerProduct{Z_t Z_t - \bXg}{\cG(U_t) \cG(U_t)^{\top}}
			- \norm{\Delta} \times \normNuclear{\cG(U_t)\cG(U_t)^{\top}}
		\end{align*}
		The second term is at most $O(r) \times \norm{E_t}$ from our proof.
		Next,
		\begin{align*}
			\innerProduct{Z_tZ_t^{\top} - \bXg}{\cG(U_t) \cG(U_t)^{\top}}
			= \innerProduct{Z_tZ_t^{\top} - \bXg}{M_t Z_tZ_t^{\top} M_t^{\top}}
			+ \innerProduct{Z_tZ_t^{\top} - \bXg}{M_t \Delta M_t^{\top}}
		\end{align*}
		The second term is at most:
		\begin{align*}
			\norm{\Delta} \times \normNuclear{M_t^{\top} (Z_tZ_t^{\top} - \bXg) M_t}
			\lesssim r \norm{E_t}
		\end{align*}
		Lastly, we expand out $M_t$ to obtain:
		\begin{align*}
			& \innerProduct{Z_tZ_t^{\top} - \bXg}{M_t Z_tZ_t^{\top} M_t^{\top}}
			= \frac 1 m \sum_{i=1}^m \innerProduct{A_i}{U_tU_t^{\top} - \bXg} \innerProduct{Z_tZ_t^{\top} - \bXg}{A_iZ_tZ_t^{\top}M_t^{\top}} \\
			& = \frac 1 m \sum_{i=1}^m \innerProduct{A_i}{Z_tZ_t^{\top} - \bXg} \innerProduct{A_i}{W} + \frac 1 m \sum_{i=1}^m \innerProduct{A_i}{\Delta} \innerProduct{A_i}{W}
		\end{align*}
		where we denote by $W = (Z_tZ_t^{\top} - \bXg)M_tZ_tZ_t^{\top}$.
		Clearly, the rank of $W$ is at most $r$ and it is not hard to see that the spectral norm of $W$ is $O(1)$.
		For the first part, we apply Lemma \ref{lem:RIP3} to obtain:
		\begin{align}
			|\innerProduct{Z_tZ_t^{\top} - \bXg}{W}| + \delta \times \normFro{Z_tZ_t^{\top} - \bXg} \normFro{W}
			\lesssim \normFro{(Z_tZ_t^{\top} - \bXg)Z_t}^2 + \normFro{Z_tZ_t^{\top} - \bXg}^2 \label{eq_conv_asym}
		\end{align}
		where we used that $\norm{M_t} \le O(1)$ and $\delta \lesssim 1 / \sqrt r$.
		For the second part, we apply Lemma \ref{lem:RIP3} and Lemma \ref{lem:RIP4} together on $\Delta$ to obtain:
		\begin{align*}
			|\innerProduct{\Delta}{W}| + \delta \times (2\normFro{E_tZ_t^{\top}} + \normNuclear{E_tE_t^{\top}}) \normFro{W}
			\le O(r) \times \norm{E_t}
		\end{align*}
		where we used the assumption that $\norm{E_t} \le 1 /d$.
		The proof for $\normFro{\cG(U_t)U_t^{\top}}^2$ is similar.
		The difference is that we will obtain $W' = Z_tZ_t^{\top} M_t Z_tZ_t^{\top}$ instead.
		To bound $\innerProduct{Z_tZ_t - \bXg}{W'}$, we use Lemma \ref{lem:RIP3} and Lemma \ref{lem:RIP4} to control $M_t$. The details are left to the readers.
	\end{proof}
		Finally we consider the degree three and four terms of $\eta$ in equation \eqref{eq_conv_deg3}:
	\begin{align*}
		\innerProduct{4\eta^3 \cG(U_t) U_t^{\top}}{\cG(U_t)\cG(U_t)^{\top}}
		\le O(\eta^3) \times \normFro{\cG(U_t)}^2,
	\end{align*}
	because $\norm{U_tU_t^{\top} - \bXg} \le O(1)$ and $\norm{\cG(U_t)} \le O(1)$. For the gradient of $U_t$, we have already decomposed it to the sum of $Y_1$, $Y_2$ and $Y_3$.
	And our proof already implies that:
	\begin{align*}
		&	\normFro{Y_1} \le O(\sqrt r) \times \norm{Z_tZ_t^{\top} - \bXg}, \mbox{ and }\\
		& \normFro{Y_2}, \normFro{Y_3} \le O(\sqrt d) \times \norm{E_t}
	\end{align*}
	Combining all results together, we get that:
	\[ \normFro{\cG(U_t)}^2 \le O(r) \norm{Z_tZ_t^{\top} - \bXg}^2 + O(d \norm{E_t}^2) \]
	Hence equation \eqref{eq_conv_deg3} is at most:
	\[ O(r \eta^3) \normFro{Z_tZ_t^{\top} - \bXg}^2 + O(\norm{E_t}) \]
	Combining the above equation with Claim \ref{claim_conv_deg1} and \ref{claim_conv_deg2}, we have shown that:
	\begin{align*}
		\normFro{U_{t+1}U_{t+1}^{\top} - \bXg}^2
		\le & \normFro{U_tU_t^{\top} - \bXg}^2 - (\eta - O(\eta^2)) \normFro{(Z_tZ_t^{\top} - \bXg)Z_t}^2 + O(\sqrt{dr}) \norm{E_t} \\
		& + O(\eta\delta + \eta^2\delta + r\eta^3) \normFro{Z_tZ_t^{\top} - \bXg}^2
	\end{align*}
	Lastly we show that:
	\begin{align}\label{eq_conv_signal}
		\normFro{(Z_tZ_t^{\top} - \bXg)Z_t}^2 \gtrsim \frac 1 {\kappa} \normFro{Z_tZ_t^{\top} - \bXg}^2.
	\end{align}
	The conclusion follows since it is not hard to show that
	\[ \normFro{U_tU_t^{\top} - \bXg}^2 = \normFro{Z_tZ_t^{\top} - \bXg}^2 \pm O(\sqrt{dr}) \norm{E_t}. \]
The rest of the proof is dedicated to equation \eqref{eq_conv_signal}.
Denote by $Z_t = U \Sigma V^{\top}$ its SVD.
Recall that $X^{\star} = U^{\star} \Sigma^{\star} {U^{\star}}^{\top}$.
We have
\begin{align*}
  \bignormFro{(Z_t Z_t^{\top} - U^{\star}) Z_t}^2
  &= \bignormFro{(U \Sigma^2  - U^{\star} \Sigma^{\star} {U^{\star}}^{\top} U) \Sigma}^2 \\
  &\ge \sigma_{\min}(\Sigma) \bignormFro{U \Sigma^2 - U^{\star} \Sigma^{\star} {U^{\star}}^{\top} U}^2 \\
  &= \sigma_{\min}(\Sigma) \left(\normFro{\Sigma^2}^2 + \normFro{\Sigma^{\star} Y}^2 - 2\innerProduct{\Sigma^2}{Y^{\top} \Sigma^{\star} Y} \right),
\end{align*}
where we denote by $Y = {U^{\star}}^{\top} U$.
Next, we have
\begin{align*}
  \bignormFro{Z_t Z_t^{\top} - X^{\star}}^2
  = \normFro{\Sigma^2}^2 + \normFro{\Sigma^{\star}}^2 - 2\innerProduct{\Sigma^2}{Y^{\top} \Sigma^{\star} Y}.
\end{align*}
By assumption we have $\sin(Z_t, U^*) \leq \frac{1}{3}$, which gives us $\sigma_{\min}(Y) \geq 1/4$.
Meanwhile, the spectral norm of $Y$ is at most 1.
Based on the two facts, we will prove equation \eqref{eq_conv_signal} by showing:
\begin{align}
  \normFro{\Sigma^2}^2 + \normFro{\Sigma^{\star} Y}^2 - 2\innerProduct{\Sigma^2}{Y^{\top}\Sigma^{\star}Y} \ge c \left(\normFro{\Sigma^2}^2 + \normFro{\Sigma^{\star}}^2 - 2\innerProduct{\Sigma^2}{Y^{\top} \Sigma^{\star} Y} \right). \label{eq_sigma_min}
\end{align}
where $c \le {\sigma_{\min}(Y)}^2 / (1 + {\sigma_{\min}(Y)}^2)$ (e.g. $c = 1/17$ suffices).
By Cauchy-Schwarz inequality,
\begin{align}
  2(1-c) \innerProduct{\Sigma^2}{Y^{\top} \Sigma^{\star} Y}
  \le (1-c) \normFro{\Sigma^2}^2 + (1-c) \normFro{Y^{\top} \Sigma^{\star} Y}^2. \label{eq_cs}
\end{align}
And then
\begin{align}
  &~ \normFro{\Sigma^{\star} Y}^2 - c \normFro{\Sigma^{\star}}^2 - (1-c) \normFro{Y^{\top} \Sigma^{\star} Y}^2 \nonumber \\
  &= (1-c)\times \trace(\Sigma^{\star} YY^{\top} \Sigma^{\star} (\Id - YY^{\top})) - c\times \trace({\Sigma^{\star}}^2 (\Id - YY^{\top})) \nonumber \\
  &\ge \left({(1-c) \sigma_{\min}(Y)^2} - c\right) \times \trace(\Sigma^{\star} \Sigma^{\star} (\Id - YY^{\top})) \ge 0, \label{eq_sigma_min_pf}
\end{align}
where the last line is because $\Id - YY^{\top}$ is PSD since $\norm{Y} \le 1$.
By combining equation \eqref{eq_cs} and \eqref{eq_sigma_min}, we have obtained equation \eqref{eq_sigma_min}.

By the assumption that $\sigma_{\min}(Z_t)^2 \geq 1/(4 { \kappa}) = \Omega(\delta)$, we complete the proof of equation \eqref{eq_conv_signal}.

\section{Missing proofs in Section~\ref{sec:quadratic}}\label{sec:proofs:q}
\begin{proof}[Proof of Lemma~\ref{lem:gaussian_concentration}]

Let us first consider the case when $X$ is a rank-1 matrix.  Suppose $X  = aa^{\top}$  with $\| a \|= 1$. We then have:
\begin{align}
\frac{1}{m} \sum_{i = 1}^m \langle A_i, X \rangle  A_i 1_{|\langle A_i, X \rangle| \leq R}&= \frac{1}{m} \sum_{i = 1}^m \langle A_i, a a^{\top} \rangle A_i \nonumber
\\
&=  \frac{1}{m} \sum_{i = 1}^m\langle x_i, a \rangle^2  x_i x_i^{\top} 1_{\langle x_i, a \rangle^2 \leq R^2}\nonumber
\end{align}
We define: 
 \begin{align}
 H(x_1, \cdots, x_m) := \sup_{u, v \in \mathbb{R}^d, \| u \|_2=  \| v\|_2 = 1} \left| \frac{1}{m}\sum_{i = 1}^m \langle x_i, u \rangle^2 \langle x_i , v \rangle^2 1_{\langle x_i,u  \rangle^2 \leq R^2} - 2 \langle u, v \rangle^2  - 1\right|\nonumber
\end{align}

It suffices to bound $H$ because by definition, for every $X  = aa^{\top}$ being a rank one matrix, with $\| a \|= 1$, we have that  
\begin{align}
\Norm{\frac{1}{m} \sum_i\langle x_i, a \rangle^2  x_i x_i^{\top} 1_{\langle x_i, a \rangle^2 \leq R^2} - 2X  - I}\leq  H(x_1, \cdots, x_m) \nonumber
 \end{align}
Let us further decompose $H$ into two terms:
\begin{align}
H(x_1, \cdots, x_m) &\leq  \sup_{u, v \in \mathbb{R}^d, \| u \|_2=  \| v\|_2 = 1} \left| \frac{1}{m}\sum_{i = 1}^m\langle x_i, u \rangle^2 \langle x_i , v \rangle^2 1_{\langle x_i,u  \rangle^2 \leq R^2 } 1_{\langle x_i,v  \rangle^2 \leq R^2 } - 2 \langle u, v \rangle^2 - 1 \right|\nonumber
\\
& + \sup_{u, v \in \mathbb{R}^d, \| u \|_2=  \| v\|_2 = 1} \left| \frac{1}{m}\sum_{i = 1}^m \langle x_i, u \rangle^2 \langle x_i , v \rangle^2   1_{\langle x_i,u  \rangle^2 \leq R^2} 1_{\langle x_i,v  \rangle^2 > R^2 }  \right|\label{eqn:split}
\end{align}

Let us bound the two term separately. For the first term, for every unit vectors $u, v$, we define functions $f_{u, v}: \mathbb{R}^d \to \mathbb{R}$ as $f_{u, v}(x) = \langle  u, x \rangle \langle v, x \rangle  1_{\langle x_i,u  \rangle^2 \leq R^2} 1_{\langle x_i,v  \rangle^2 \leq R^2 }$. We have that $f_{u, v}(x) \leq R^4$. Thus by the symmetrization technique and the contraction principle(e.g.,  see Corollary 4.7 in~\cite{adamczak2010quantitative}), we have:
\begin{align}
\E\left[ \sup_{u, v \in \mathbb{R}^d, \| u \|_2=  \| v\|_2 = 1} \left| \sum_{i  = 1}^m \left(f_{u, v}(x_i)^2 - \E\left[f_{u, v}(x_i)^2\right] \right) \right|\right] \leq 8R^2\E\left[  \sup_{u, v \in \mathbb{R}^d, \| u \|_2=  \| v\|_2 = 1} \left| \sum_{i = 1}^m \veps_i f_{u, v}(x_i)\right| \right]\nonumber
\end{align}
 where $\{\veps_i\}_{i = 1}^m$ is a set of i.i.d.  Rademacher random variables.  We can further bound the right hand side of the inequality above by:
 \begin{align}
 \E\left[  \sup_{u, v \in \mathbb{R}^d, \| u \|_2=  \| v\|_2 = 1} \left| \sum_{i = 1}^m \veps_i f_{u, v}(x_i)\right| \right] &\leq \E\left[  \Norm{\sum_{i = 1}^m \veps_i x_i x_i^{\top}}\right]\nonumber
 \end{align}
 A standard bound on the norm of Gaussian random matrices gives us: $\E\left[  \Norm{\sum_{i = 1}^m \veps_i x_i x_i^{\top}}\right] \lesssim \sqrt{md}$. Therefore, we conclude that 
 \begin{align}
 \E\left[ \sup_{u, v \in \mathbb{R}^d, \| u \|_2=  \| v\|_2 = 1} \left| \sum_{i  = 1}^m \left(f_{u, v}(x_i)^2 - \E\left[f_{u, v}(x_i)^2 \right] \right) \right|\right]  \lesssim R^2 \sqrt{md}\nonumber
 \end{align}
 Now let us consider the expectation of $f_{u, v}(x_i)^2$, a direct calculation shows that 
 \begin{align}
\left| \E[f_{u, v}(x_i)^2]  - 2\langle u, v \rangle^2 - 1 \right| &=\left| \E\left[\langle x_i, u \rangle^2 \langle x_i , v \rangle^2 1_{\langle x_i,u  \rangle^2 \leq R^2 } 1_{\langle x_i,v  \rangle^2 \leq R^2 }\right] - \E\left[\langle x_i, u \rangle^2 \langle x_i , v \rangle^2  \right]\right|\nonumber
\\
& \leq \E\left[\langle x_i, u \rangle^2 \langle x_i , v \rangle^2 1_{\langle x_i,u  \rangle^2 \leq R^2 } 1_{\langle x_i,v  \rangle^2 > R^2 }\right]\nonumber
+ \E\left[\langle x_i, u \rangle^2 \langle x_i , v \rangle^2 1_{\langle x_i,u  \rangle^2 > R^2 } \right]\nonumber
\\
& \leq 2 \E[\langle x_i, u \rangle^2 \langle x_i , v \rangle^2 1_{\langle x_i,u  \rangle^2 > R^2 } ]\nonumber
 \end{align}
An elementary calculation of Gaussian variables gives us: 
\begin{align}
 \E[\langle x_i, u \rangle^2 \langle x_i , v \rangle^2 1_{\langle x_i,u  \rangle^2 > R^2 } ] \lesssim R^4 e^{- R^2/2}
\end{align}
\noindent Putting everything together, for $R \geq 1$ we are able to bound the first term of equation~\eqref{eqn:split} by:
{\small 
\begin{align}
\E\left[ \sup_{u, v \in \mathbb{R}^d, \| u \|_2=  \| v\|_2 = 1} \left| \frac{1}{m}\sum_{i = 1}^m\langle x_i, u \rangle^2 \langle x_i , v \rangle^2 1_{\langle x_i,u  \rangle^2 \leq R^2 } 1_{\langle x_i,v  \rangle^2 \leq R^2 } - 2 \langle u, v \rangle^2 - 1 \right|  \right]\lesssim R^4 \left( \sqrt{\frac{d}{m}}+ e^{- R^2 /2} \right) \label{eqn:211}
\end{align}}
Moreover, for every $u, v \in \mathbb{R}^d, \| u \|=  \| v\| = 1$ we know that $f_{u, v}(x)  \leq R^2$, we can apply~\cite[Lemma 4.8]{adamczak2010quantitative} to transform the bound above into a high probability bound. We have that for every $s \in [0, 1]$, with probability at least $1 - e^{-\Omega(s^2 m / (dR^4))}$, 
{\small 
\begin{align}
 \sup_{u, v \in \mathbb{R}^d, \| u \|_2=  \| v\|_2 = 1} \left| \frac{1}{m}\sum_{i = 1}^m\langle x_i, u \rangle^2 \langle x_i , v \rangle^2 1_{\langle x_i,u  \rangle^2 \leq R^2 } 1_{\langle x_i,v  \rangle^2 \leq R^2 } - 2 \langle u, v \rangle^2 - 1 \right| \lesssim R^4 \left( \sqrt{\frac{d}{m}}+ e^{- R^2 /2} \right) + s \nonumber
\end{align}
}
Picking $s = \left(\frac{R^2 \sqrt{d}}{\sqrt{m}} \log \frac{1}{q}  \right)$ with $R =  \Theta \left(\log\left(\frac{1}{ \delta}\right) \right)$ we obtain that 
\begin{align}
\sup_{u, v \in \mathbb{R}^d, \| u \|_2=  \| v\|_2 = 1} \left| \frac{1}{m}\sum_{i = 1}^m\langle x_i, u \rangle^2 \langle x_i , v \rangle^2 1_{\langle x_i,u  \rangle^2 \leq R^2 } 1_{\langle x_i,v  \rangle^2 \leq R^2 } - 2 \langle u, v \rangle^2 - 1 \right| \leq \delta\nonumber
\end{align}
For the second term of equation~\eqref{eqn:split} , we have that 
\begin{align}
& \sup_{u, v \in \mathbb{R}^d, \| u \|_2=  \| v\|_2 = 1} \left| \frac{1}{m}\sum_{i = 1}^m \langle x_i, u \rangle^2 \langle x_i , v \rangle^2   1_{\langle x_i,u  \rangle^2 \leq R^2} 1_{\langle x_i,v  \rangle^2 > R^2 }  \right|\nonumber \\
 & \leq  \sup_{ v \in \mathbb{R}^d, \| v\|_2 = 1} \left| \frac{1}{m}\sum_{i = 1}^m R^2 \langle x_i , v \rangle^2   1_{\langle x_i,v  \rangle^2 > R^2 }  \right|\nonumber
\end{align}
By \cite[Theorem 3.6 and Remark 3.10]{adamczak2010quantitative}, we have that for every $s > 0$, with probability at $1 - e^{-\Omega(s \sqrt{d})}$:
\begin{align}
\sup_{ v \in \mathbb{R}^d, \| v\|_2 = 1} \left| \frac{1}{m}\sum_{i = 1}^m R^2 \langle x_i , v \rangle^2   1_{\langle x_i,v  \rangle^2 > R^2 }  \right| \lesssim R^2 s^2 \left( \frac{d^2}{m} + \frac{d^2}{m} s^2 R^{-2} \log^2 \frac{m}{n} \right)
\end{align}

Taking $s =\Omega\left( \frac{ \log \frac{1}{q}}{\sqrt{d}}\right)$, putting everything together we prove the Lemma for the case when $\bX=aa^\top$ is rank one. 

For $\bX$ of general rank, the proof follows by decomposing $\bX$ to a sum of rank one singular vectors and apply triangle inequality directly. 
\end{proof}

\section{Restricted Isometry Properties}\label{sec:rip}

In this section we list additional properties we need for the set of measurement matrices $\{A_i\}_{i=1}^m$.
Lemma \ref{lem:RIP3} follows from the definition of RIP matrices.
The rest three Lemmas are all direct implications of Lemma \ref{lem:RIP3}.
\begin{proof}[Proof of Lemma~\ref{lem:property_1}]
For every $x \in \mathbb{R}^d, y \in \mathbb{R}^{d'}$ of norm at most $1$,
we have: 
\begin{align*}
	\frac 1 m \sum_{i = 1}^m \langle \bA_i , \bX \rangle x^{\top}\bA_i \bR y - x^{\top}\bX \bR y
	&= \frac 1 m \sum_{i = 1}^m \langle \bA_i , \bX \rangle \langle \bA_i, x y^{\top}\bR^{\top} \rangle - x^{\top}\bX \bR y \\
	& \leq \langle \bX, xy^{\top} \bR^{\top} \rangle  + \delta \| \bX \|_F \norm{xy^{\top}\bR^{\top}} - x^{\top}\bX \bR y \\
	& \leq \delta \| \bX \|_F \|\bR\|_2
\end{align*}
The first inequality uses Lemma \ref{lem:RIP3}.
\end{proof}

The following Lemmas deal with matrices that may have rank bigger than $r$.
The idea is to decompose the matrix into a sum of rank one
matrices via SVD, and then apply Lemma \ref{lem:RIP3}.

\begin{lem}\label{lem:RIP4}
	Let $\{A_i\}_{i=1}^m$ be a family of matrices in $\Real^{d \times d}$
	that satisfy $(r, \delta)$-restricted isometry property.
	Then for any matrices $X, Y \in \Real^{d \times d}$,
	where the rank of $Y$ is at most $r$,
	we have:
	\[ \bigabs{\frac 1 m \sum_{i=1}^m \innerProduct{A_i}{X} \innerProduct{A_i}{Y} 		- \innerProduct{X}{Y} }
		\le \delta \normNuclear{X} \normFro{Y} \]
\end{lem}

\begin{proof}
	Let $X = U D V^{\top}$ be its SVD.
	We decompose $D = \sum_{i=1}^d D_i$ where each $D_i$ contains only the i-th	diagonal entry of $D$,
	and let $X_i = U D_i V^{\top}$ for each $i=1,\dots,d$.
	Then we have:
	\begin{align*}
		\frac 1 m \sum_{i=1}^m \innerProduct{A_i}{X} \innerProduct{A_i}{Y}
		&= \sum_{j=1}^d \left( \frac 1 m \sum_{i=1}^m \innerProduct{A_i}{X_j} \innerProduct{A_i}{Y} \right) \\
		&\le \sum_{j=1}^d \left(\innerProduct{X_j}{Y} + \delta \normFro{X_j}\normFro{Y} \right) = \innerProduct{X}{Y} + \delta \normNuclear{X} \normFro{Y}
	\end{align*}
\end{proof}

\begin{lem}\label{lem:property_2}
	Let $\{A_i\}_{i=1}^m$ be a family of matrices in $\Real^{d \times d}$
	that satisfy $(1, \delta)$-restricted isometric property.
	Then for any matrix $X \in \Real^{d \times d}$ and matrix $R \in \Real^{d \times d'}$, where $d'$ can be any positive integer,
	we have:
	\[\bignorm{ \frac{1}{m}\sum_{i = 1}^m \innerProduct{A_i}{X}\, A_i R - X R} \leq
	\delta \normNuclear{X} \times \norm{R}. \]
	The following variant is also true:
	\[ \bignorm{\frac 1 m \sum_{i=1}^m \innerProduct{A_i}{X}\, U A_i R - U X R} \le
	\delta \normNuclear{X} \times \norm{U} \times \norm {R}, \]
 	where $U$ is any matrix in $\Real^{d \times d}$.
\end{lem}

\begin{proof}
	Let $X = U D V^{\top}$ be its SVD.
	We define $X_i$ and $D_i$ the same as in the proof of Lemma \ref{lem:RIP4},
	for each $i=1,\dots,d$.
	
	For every $x \in \Real^d$, $y \in \Real^{d'}$ with norm at most one, we have:
	\begin{align*}
		&\frac 1 m \sum_{i=1}^m \innerProduct{A_i}{X}\, x^{\top} A_i R y
			- x^{\top}\bX\bR y \\
		=& \sum_{j=1}^d \left( \frac 1 m \sum_{i=1}^m \innerProduct{\bA_i}{\bX_j} \innerProduct{\bA_i}{x y^{\top}\bR^{\top}} \right)
			- x^{\top}\bX\bR y \\
		\le& \sum_{j=1}^d \left( \innerProduct{\bX_j}{x y^{\top}\bR^{\top}} + \delta \normFro{\bX_j}\norm{\bR} \right)
			- x^{\top}\bX\bR y
		= \delta \normNuclear{\bX}\norm{\bR}.
	\end{align*}
	The variant can be proved by the same approach (details omitted).
\end{proof}

\noindent {\bf Asymmetric sensing matrices.}
Recall that when each $A_i$ is asymmetric, we simply use $(A_i + A_i^{\top}) / 2$ instead of $A_i$ as our sensing matrix.
While $\{(A_i + A_i) / 2\}_{i=1}^m$ may only ensure the restricted isometry property for symmetric matrices,
we have the same inequality when the matrix $X$ in Lemma \ref{lem:property_1} is symmetric, which is the case for all our applications of Lemma \ref{lem:property_1}:
\footnote{More precisely, $X$ corresponds to any one of $U_tU_t^{\top} - \bXg$, $E_tE_t^{\top}, Z_tZ_t^{\top}$, or their linear combinations.}
\begin{align*}
	\bignorm{\frac 1 m \sum_{i=1}^m \innerProduct{\frac {A_i + A_i^{\top}} 2}{X} (A_i + A_i^{\top}) R / 2 - XR} \le \delta \normFro{X} \norm{R}.
\end{align*}
Since $X$ is symmetric, $\innerProduct{A_i}{X} = \innerProduct{A_i^{\top}}{X}$.
The above equation then follows by applying Lemma \ref{lem:property_1} twice,
with $\{A_i\}_{i=1}^m$ and $\{A_i^{\top}\}_{i=1}^m$ as sensing matrices respectively.

For the applications of Lemma \ref{lem:RIP3} in Equations \eqref{eqn:100}, \eqref{eqn:101}, \eqref{eq_error2_p2}, \eqref{eq_error2_p3} and \eqref{eq_conv_asym},
we note that either X or Y is symmetric in all applications.
Suppose that $X$ is symmetric, then we have the following
when we use $(A_i + A_i^{\top}) / 2$ as the $i$-th sensing matrix:
\begin{align*}
	\bigabs{\frac 1 m \sum_{i=1}^m \innerProduct{\frac{A_i + A_i^{\top} } 2}{X} \innerProduct{\frac{A_i + A_i^{\top}} 2} {Y} - \innerProduct{X}{Y}}
	\le \delta \normFro{X} \bignormFro{\frac{Y + Y^{\top}} 2}
\end{align*}
It is straightforward to verify that our proof still holds using the above inequality instead. The details for left for the readers.

\bibliographystyle{plain}
\bibliography{ref}

\end{document}